\documentclass[11pt]{article} %
\usepackage{fullpage}

\usepackage{authblk}
\usepackage[utf8]{inputenc} %
\usepackage[T1]{fontenc}    %
\usepackage{hyperref}       %
\usepackage{url}            %
\usepackage{booktabs}       %
\usepackage{amsfonts}       %
\usepackage{nicefrac}       %
\usepackage{microtype}      %
\usepackage{xcolor}         %
\usepackage{titletoc}

\usepackage{graphicx} %
\usepackage{natbib}
\usepackage{amsmath}
\usepackage{amssymb}
\usepackage{amsthm}
\usepackage{tikz}
\usetikzlibrary{automata,arrows, positioning}
\usepackage{algorithm}
\usepackage{algorithmicx}
\usepackage{algpseudocode}
\usepackage{mathtools}
\mathtoolsset{showonlyrefs=true}

\def\gA{{\mathcal{A}}}
\def\gB{{\mathcal{B}}}
\def\gC{{\mathcal{C}}}
\def\gD{{\mathcal{D}}}

\def\gG{{\mathcal{G}}}
\def\gH{{\mathcal{H}}}

\def\gM{{\mathcal{M}}}

\def\gO{{\mathcal{O}}}

\def\gS{{\mathcal{S}}}

\def\gX{{\mathcal{X}}}

\usepackage{pifont}
\newcommand{\greentick}{\textcolor{green}{\ding{51}}}
\newcommand{\redcross}{\textcolor{red}{\ding{55}}}
\newcommand{\abs}[1]{\left| {#1} \right|}

\def\A{{\mathcal{A}}}
\newcommand{\lp}{\left(}
\newcommand{\rp}{\right)}
\newcommand{\lb}{\left\{}
\newcommand{\rb}{\right\}}
\newcommand{\ls}{\left[}
\newcommand{\rs}{\right]}
\newcommand{\labs}{\left|}
\newcommand{\rabs}{\right|}
\newcommand{\lnorm}{\left\|}
\newcommand{\rnorm}{\right\|}
\newcommand{\expect}{\mathbb{E}}
\def\TV{\mathrm{TV}}
\newcommand{\Nashgap}{\mathrm{Nash-}\mathrm{Gap}}

\definecolor{greenp}{rgb}{0.0, 0.51, 0.5}

\newcommand{\initial}{d_0}
\newcommand{\bcc}[1]{\left\{{#1}\right\}}
\newcommand{\brr}[1]{\left({#1}\right)}
\newcommand{\bs}[1]{\left[{#1}\right]}
\newcommand{\norm}[1]{\left\| {#1} \right\|}

\newcommand{\innerprod}[2]{\left\langle{#1},{#2}\right\rangle}
\usepackage{dsfont}
\usepackage{thm-restate}

\newcommand{\argmin}{\mathop{\rm argmin}}

\newcommand{\expert}{\operatorname{E}}

\newcommand{\muE}{\mu^{\expert}}
\newcommand{\nuE}{\nu^{\expert}}
\newcommand{\ours}{MAIL-WARM}

\newtheorem{definition}{Definition}[section]
\newtheorem{theorem}{Theorem}[section]
\newtheorem{lemma}{Lemma}[section]

\newtheorem{corollary}{Corollary}[section]

\newcommand{\authorblockoneaffil}[3]{%
  \begin{minipage}[t]{0.25\textwidth} %
    \centering\strut
    #1 \\
    #2 \\
    #3 \strut
  \end{minipage}%
}
\usepackage{amsmath} 

\title{Rate optimal learning of equilibria from data}

\begin{document}
  \date{}

\begin{center}
  \makeatletter %
  {\Large \@title\par} %
  \makeatother
  \vspace{1.5em} %
\end{center}
\begin{center}
    
\authorblockoneaffil{Till Freihaut*}{freihaut@ifi.uzh.ch}{University of Zurich}%
  \hspace{0.03\textwidth}%
  \authorblockoneaffil{Luca Viano*}{luca.viano@epfl.ch}{EPFL}%
  \hspace{0.03\textwidth}%
  \authorblockoneaffil{Emanuele Nevali}{emanuele.nevali@epfl.ch}{EPFL} %
  \hspace{0.03\textwidth}\\
  \vspace{1mm}%
  \authorblockoneaffil{Volkan Cevher}{volkan.cevher@epfl.ch}{EPFL} 
  \hspace{0.03\textwidth}
    \authorblockoneaffil{Matthieu Geist}{matthieu@earthspecies.org}{Earth Species Project}%
  \hspace{0.03\textwidth}%
  \authorblockoneaffil{Giorgia Ramponi}{ramponi@ifi.uzh.ch}{University of Zurich}
  \end{center}

\begingroup %
\renewcommand\thefootnote{$\ast$}
\footnotetext{Equal contribution, alphabetical order.}
\endgroup

\begin{abstract}
We close open theoretical gaps in Multi-Agent Imitation Learning (MAIL) by characterizing the limits of non-interactive MAIL and presenting the first interactive algorithm with near-optimal sample complexity.
In the non-interactive setting, we prove a statistical lower bound that identifies the \emph{all-policy deviation concentrability coefficient} as the fundamental complexity measure, and we show that Behavior Cloning (BC) is rate-optimal. For the interactive setting, we introduce a framework that combines reward-free reinforcement learning with interactive MAIL and instantiate it with an algorithm, \emph{\ours}. It improves the best previously known sample complexity from $\gO(\varepsilon^{-8})$ to $\gO(\varepsilon^{-2}),$ matching the dependence on $\varepsilon$ implied by our lower bound. Finally, we provide numerical results that support our theory and illustrate, in environments such as grid worlds, where Behavior Cloning fails to learn.
\end{abstract}

\section{Introduction}
More and more AI systems are deployed in real-world scenarios. This naturally leads to AI systems interacting and adapting their behavior to each other. Importantly, this interaction can be captured as a multi-agent system \citep{CAIF_1} %
, and since reward functions are often inaccessible, learning directly from expert demonstrations via Imitation Learning (IL) becomes especially compelling. To capture expert behavior without knowing a reward function, IL  serves as a great framework, showcasing impressive empirical success in single-agent settings \citep{torabi2019recent, jain2025smoothseaskilledtextttsailor, foster2024behavior} and strong theoretical guarantees \citep{foster2024behavior, viano2024imitation, rajaraman2020toward}. However, applying IL to multi-agent systems remains largely underexplored. In particular, previous works \citep{pmlr-v97-yu19e, NEURIPS2018_240c945b, NEURIPS2024_2fbeed1d} are mostly empirical and they lack theoretical guarantees. Moreover, they often fail to capture the potentially strategic behavior of agents acting in multi-agent systems since they are not designed to learn a Nash equilibrium profile.\\

 Closer to our work,  \citet{tang2024multiagentimitationlearningvalue} showed that optimizing the objective that captures these strategic behaviors,  namely the \emph{Nash Gap}, is hard. In particular, they provide guarantees for BC, assuming that the equilibrium profile generating the data assigns strictly positive probability to every state. More recently, \citet{freihaut2025learningequilibriadataprovably} dropped this assumption and proved a tighter BC guarantee involving the \emph{all policy} deviation concentrability coefficient $\mathcal{C}_{\max}$. On an intuitive level, $\mathcal{C}_{\max}$ quantifies the coverage only of the states that can be reached by a best response against an arbitrary policy. Their work considers two settings. In the non-interactive setting, the learner receives a fixed dataset of expert trajectories and cannot query the expert further. In the interactive setting, the learner is allowed to query the expert at states encountered during training. In their work and also this work, the experts are assumed to be playing according to a Nash equilibrium strategy $(\muE,\nuE).$ For the non-interactive case, they showed that dependence on $\mathcal{C}(\muE,\nuE)$ is unavoidable. Informally, $\mathcal{C}(\muE,\nuE)$ describes the coverage of states reachable under all potential Nash equilibria. If $\mathcal{C}(\muE,\nuE)$ is infinite, then no algorithm can succeed, even with unlimited data. They removed the $\mathcal{C}(\muE,\nuE)$ dependence in the interactive setting, introducing MURMAIL that comes with a sample complexity of order $\gO(\varepsilon^{-8})$ independent of $\mathcal{C}(\muE,\nuE)$. \\

 However, the authors left open several questions, which we address in this work. The first concerns the non-interactive setting and is presented next.
 \vspace{-1mm}

 \begin{center}
\textbf{Open Question 1}\emph{ Does there exist a non-interactive MAIL algorithm with guarantees featuring only $\mathcal{C}(\muE,\nuE)$ and not $\mathcal{C}_{\max}$ ? }
 \end{center}
Answering this question is crucial for both theoreticians and practitioners, as it clarifies when applying BC as an algorithm is appropriate and when, instead, an interactive expert is required. Moreover, in the interactive setting, the guarantees for MURMAIL scales suboptimally with the game parameters and precision $\varepsilon$. Therefore, it is natural to ask:
 \begin{center}
     \textbf{Open Question 2} \emph{Can we design an algorithm which outputs an $\varepsilon$-approximate Nash equilibrium with the optimal order of expert queries, which is  $\mathcal{O}(\varepsilon^{-2})$ ?}
 \end{center}

In this work, we answer these questions with the following contributions:

\begin{enumerate}
    \item %
    We construct a Markov game where, even if $\gC(\muE,\nuE)$ is bounded, no non-interactive Imitation Learning algorithm can learn an $\varepsilon$-Nash equilibrium from data. Surprisingly, the construction is a striking simple Markov Games with $3$ states only.
    \item Additionally, with the same construction, we show that the \emph{all policy deviation concentrability coefficient} $\mathcal{C}_{\max}$, which upper bounds $\mathcal{C}(\muE,\nuE),$ is the fundamental quantity of non-interactive MAIL by proving a statistical lower bound of order $\Omega( \mathcal{C}_{\max} \varepsilon^{-2})$ on the sample complexity of any non-interactive MAIL algorithm. This construction answers \textbf{Open Question 1} in the negative.
    Moreover, since \citet{freihaut2025learningequilibriadataprovably} proved an upper bound for BC of order $\mathcal{O}\brr{\gC^2_{\max} \varepsilon^2}$, we conclude that BC is rate optimal in the non-interactive setting. Indeed, BC matches the optimal $\varepsilon$-dependence, and the gap between the BC upper bound and the information-theoretic lower bound is only polynomial in the concentrability and in the parameters of the game.
    \item For the interactive setting, we provide a new framework, that combines reward-free exploration with (interactive) Multi-Agent Imitation Learning.
    This allows us to derive a new algorithm, namely \emph{\ours}, that improves the best currently known sample complexity of MURMAIL from $\gO(\varepsilon^{-8})$ to $\gO(\varepsilon^{-2})$, matching the lower bound in $\varepsilon$ in this setting. This provides a positive answer to \textbf{Open Question 2}.
    \item We empirically demonstrate the effectiveness of our algorithm, showing that it outperforms other interactive Multi-Agent Imitation Learning algorithms in settings where BC fails to recover an $\varepsilon$-Nash equilibrium.
\end{enumerate}
A complete summary of our results can be found in Table~\ref{tab:literature_MAIL}.
\begin{table*}[t]
    \caption{\label{tab:literature_MAIL} For simplicity, we report results for the two-player zero-sum setting with horizon $H$, finite state space $\gS$, finite action spaces $\A$, $\mathcal{B}$. Let $A_{\max} = \max{(\abs{\A}, \abs{\mathcal{B}})}$.  For a fair comparison, we restate the results of \cite{freihaut2025learningequilibriadataprovably} in the finite-horizon setting. Indeed, we have verified that the results prove therein transfer to the finite-horizon setting. 
    The column  $\textbf{Expert Data}$ reports the number of data collected in either the interactive or non-interactive setting to attain a Nash gap bound of order $\mathcal{O}(\varepsilon)$.
    }\vspace{2mm}\centering
    \resizebox{0.9\textwidth}{!}{%
    
    \begin{tabular}{|c|c|c|}
        \hline
        \textbf{Algorithm}  %
        &  \textbf{Expert Data} & \textbf{Queriable Expert}\\ \hline
        BC (Analysis in \cite{freihaut2025learningequilibriadataprovably})   %
        &  $\widetilde{\mathcal{O}}\brr{\frac{H^4SA_{\max} \mathcal{C}^2_{\max}}{\varepsilon^2}}$  &  \redcross \\ \hline
\textbf{Lower Bound (This work)} & $\Omega\brr{\frac{\gC_{\max}}{\varepsilon^2}}$ & \redcross
        \\ \hline
        MURMAIL (\cite{freihaut2025learningequilibriadataprovably})   %
        &  $\widetilde{\mathcal{O}}\brr{ \frac{H^{12}S^4A_{\max}^5}{\varepsilon^8}}$  &  \greentick \\ \hline 
        \textbf{\ours (This work)}   %
        &  $\widetilde{\gO}\brr{\frac{H^7S^3A_{\max}^3}{\varepsilon^2}} $ &  \greentick \\ \hline
        \textbf{Lower bound (This work)} & $\Omega \brr{\varepsilon^{-2}}$ & \greentick \\\hline
    \end{tabular}}
\end{table*}

\section{Preliminaries}
\label{sec:Preliminaries}
We start by formalizing the concept of two-player zero-sum Markov games.

\paragraph{Two-player zero-sum Markov game}
A finite-horizon two-player zero-sum Markov game is described by the tuple $\gG = (H, \gS, \gA, \gB, P, r, \initial)$, where $\gS$ is a finite state space of cardinality $S := \abs{\gS}$, $\gA$ and $\gB$ are the finite action spaces of cardinality $A := \abs{\gA}$ and $B := \abs{\gB}$  for player 1 and player 2, respectively. Moreover, let $A_{\max} = \max \bcc{A, B}$ be the cardinality of the largest action space. The transition dynamics at each time step $h \in \{1, \ldots, H\} := [H]$ are governed by an (unknown) transition kernel $P_h \in \mathbb{R}^{SAB \times S}$, and the reward function $r_h \in [-1,1]^{SAB}$ assigns a scalar payoff to each state-action triplet. The game starts from an initial state $S_0 \sim \initial$, where $\initial$ is a distribution over $\gS$.\looseness=-1\\

In this setup, player 1 seeks to maximize the total reward, while player 2 aims to minimize it, leading to the zero-sum property: for any $(s,a,b) \in \gS \times \gA \times \gB$ and all $h \in [H]$, it holds that $r_h^1(s,a,b) = -r_h^2(s,a,b)$. Hence, we omit superscripts and refer to the reward as $r$. A (stochastic) Markov policy for player 1 is denoted by $\mu_h: \gS \times H \to \Delta_{\gA}$, and for player 2 by $\nu_h: \gS \times H \to \Delta_{\gB}$, where $\Delta_{\gA}$ and $\Delta_{\gB}$ denote probability simplices over actions. Moreover, we will use $\Pi_{\mu}$, $\Pi_{\nu}$ to denote the set of all Markov policies for the $\mu$ and $\nu$ player respectively. We denote a policy $\mu$ of a game as a set of policies $\mu:=\{\mu_h\}_{h=0}^{H}$ and similar for $\nu :=\{\nu_h\}_{h=0}^{H}.$ \\

Let $\{(S_h, A_h, B_h)\}_{h=0}^{H}$ denote the stochastic process induced by a policy pair $(\mu, \nu)$ acting in a Markov game $\mathcal{G}$. Then, the value function and state-action value function are defined as:
\begin{align*}
    &V_h^{\mu, \nu}(s) := \mathbb{E}_{\mu, \nu}\left[ \sum_{t = h}^{H} r_t(S_t, A_t, B_t) \,\middle|\, S_h = s \right],\\
    &Q_h^{\mu, \nu}(s,a,b) := \mathbb{E}_{\mu, \nu}\left[ \sum_{t = h}^{H} r_t(S_t, A_t, B_t) \,\middle|\, S_h = s, A_h = a, B_h = b \right].
\end{align*}
We also define the (unnormalized) state visitation distribution at stage $h \in [H]$ induced by the policy pair $(\mu, \nu)$ as follows:
\[
d_h^{\mu, \nu}(s) := \mathbb{E}_{\mu, \nu}\left[ \mathds{1}_{\bcc{S_h = s}} \,\middle|\, S_0 \sim \initial \right].
\]

Fixing the policy of one agent reduces the Markov game to a Markov decision process (MDP). For example, if player 2 follows policy $\nu$, then the effective transition dynamics under $(s,a)$ are:
\[
P^{\nu}_h(s' \mid s,a) := \sum_{b \in \gB} \nu_h(b \mid s) P_h(s' \mid s,a,b).
\]
A similar expression holds when $\mu$ is fixed. The best-response set for each player against a fixed policy of the other is defined as:
\[
\mathrm{br}(\nu) := \arg\max_{\mu \in \Pi_{\mu}} \langle \initial, V_0^{\mu,\nu} \rangle.~~
\]

Equivalently for the second player, we have $\mathrm{br}(\mu) := \arg\min_{\nu \in \Pi_{\nu}} \langle \initial, V_0^{\mu,\nu} \rangle.$
Notice, that the best-response may not be unique, therefore $\mathrm{br}(\mu)$ and $\mathrm{br}(\nu)$ are sets in general. On the contrary,  the corresponding value in zero-sum games is unique. A pair of policies $(\mu^\star, \nu^\star)$ forms a Nash equilibrium (NE) if they are best responses to each other.\\

To quantify how far a policy pair is from an equilibrium, we define the \emph{Nash gap}:
\begin{equation}
\label{eq:Nash_gap}
    \Nashgap(\mu,\nu) := \langle \initial, V_0^{\mu^\star, \nu} - V_0^{\mu, \nu^\star} \rangle.
\end{equation}
This measure satisfies $\Nashgap(\mu,\nu) = 0$ if and only if $(\mu,\nu)$ is a NE, and is strictly positive otherwise.
Another important property of zero-sum games is that the set of Nash equilibria is convex.

\section{Setting and Main Results}

In Multi-Agent Imitation Learning (MAIL), the objective is to design an algorithm $\mathrm{Alg}$ that, after accessing $\mathrm{poly}(\varepsilon^{-1})$ actions sampled from the expert policies, returns a pair of policies $(\widehat{\mu}, \widehat{\nu})$ such that the expected Nash gap is bounded:
\begin{equation}
\mathbb{E}_{\mathrm{Alg}}\left[\Nashgap(\widehat{\mu}, \widehat{\nu})\right] < \varepsilon. \label{eq:learning_goal}
\end{equation}
This formulation captures the goal of learning approximately optimal behavior in competitive settings through selective expert guidance.\\

Furthermore, we differentiate the two following settings:
\begin{itemize}
 \item  We refer to \textbf{non-interactive} Multi-Agent Imitation Learning as the setting where a dataset is precollected from Nash equilibrium policies $(\muE,\nuE)$. In particular, states are sampled as $\bcc{s^i_h}^N_{i=1} \sim \rho_h \in \Delta_{\gS},$ while actions are sampled as $\bcc{a^i_h}^N_{i=1} \sim \muE(\cdot|s^i_h)$ and $\bcc{b^i_h}^N_{i=1} \sim \nuE(\cdot|s^i_h).$ Once the dataset is received, the learner can no further interact with the expert policies during the learning process.
 \vspace{-1mm}
 \item In \textbf{interactive} Multi-Agent Imitation Learning, the learner can query the expert on demand. The learning process unfolds over multiple rounds of interaction with the environment. During each round, the learner deploys a policy pair to collect a trajectory and may query the expert at any visited state.
\end{itemize}
\vspace{-2mm}
Non-interactive MAIL has the advantage of avoiding a queriable expert, but in that setting, the theoretical bounds are worse.
In contrast, interactive Imitation Learning algorithms can achieve statistical bounds independent of $\mathcal{C}(\muE, \nuE)$, but this setting captures a smaller subset of real-world scenarios, as a queriable expert might not always be available.
\subsection{Main result in non-interactive MAIL}
  In this section, we present our main result that solves a question left open by \citep{freihaut2025learningequilibriadataprovably}. We consider the finite-horizon setting, which can be related to the discounted case through the effective horizon $H \approx \frac{1}{1-\gamma}$. In order to state the theoretical gap we need to give some context which we introduce next. We adapt the concentrability coefficient from \cite{freihaut2025learningequilibriadataprovably} to the finite-horizon setting. For a policy pair $\mu, \nu$ and dataset state distribution $\rho := \bcc{\rho_h}^H_{h=1}$ it is defined as
\begin{align}
C(\mu,\nu):= \max &\left\{\max_{\nu^\star \in \mathrm{br}(\mu)}\max_{h\in[H]} \norm{\frac{d_h^{\muE,\nu^\star}}{\rho_h}}_{\infty},\max_{\mu^\star \in \mathrm{br}(\nu)}\max_{h\in[H]} \norm{\frac{d_h^{\mu^\star,\nuE}}{\rho_h}}_{\infty}\right\}.
\end{align}
In most common non-interactive situations, we have $\rho_h = d^{\muE,\nuE}_h$ where $\muE,\nuE$ is a possible Nash equilibrium profile. 
They showed that if $C(\muE,\nuE) = \infty$ any non-interactive Multi-Agent Learning algorithm suffers from a Nash gap of the order of the horizon $H$. Defining $\gC_{\max}=\max_{\mu,\nu}\gC(\mu,\nu)$ , \cite{freihaut2025learningequilibriadataprovably} also presented a behavioral cloning analysis showing that $\widetilde{\mathcal{O}}\brr{\gC^2_{\max} \varepsilon^{-2}}$ samples are needed to learn a policy pair $\hat{\mu},\hat{\nu}$ achieving the learning goal given in \eqref{eq:learning_goal}. Clearly, we have that $\gC_{\max} \geq \gC(\muE,\nuE)$, therefore there is a gap between the upper and the lower bound. In particular, given only the results of \cite{freihaut2025learningequilibriadataprovably} it is not clear if BC or another non-interactive MAIL algorithm can learn when $\gC(\muE,\nuE)$ is finite while $\gC_{\max}$ is infinite. The following result closes the gap in the negative excluding the possibility that a non-interactive MAIL algorithm can avoid the dependence on $\gC_{\max}$ in its sample complexity.

\begin{theorem}
\label{thm:lower_bound}
    Let $\hat{\mu},\hat{\nu}$ be the output of a non-interactive MAIL algorithm $\mathrm{Alg}$. Then, for any  $\mathrm{Alg}$, there exists a Markov game such that satisfying $\mathbb{E}_{\mathrm{Alg}}\bs{\innerprod{\initial}{V^{\mu^{\star} , \widehat{\nu} } -   V^{ \widehat{\mu}, \nu^{\star} }}} \leq \gO(\varepsilon)$ requires an expert dataset of size $N=\Omega(\frac{\gC_{\max}}{\varepsilon^2})$.
\end{theorem}
The same construction gives the following corollary for unbounded $\mathcal{C}_{\max}$.
\begin{corollary}\label{cor:lower}
    For any non-interactive MAIL algorithm $\mathrm{Alg}$, there exists a Markov game $\gG$ with $\gC(\muE,\nuE) <\infty$ and $\gC_{\max} = \infty$ where $\mathbb{E}_{\mathrm{Alg}}\bs{\innerprod{\initial}{V^{\mu^{\star} , \widehat{\nu} } -   V^{ \widehat{\mu}, \nu^{\star} }}} \geq \frac{H-1}{60}$.
\end{corollary}
The lower bound therefore excludes the existence of non-interactive MAIL algorithms improving the dependence on $\varepsilon$ on the sample complexity bound for BC and that avoids the dependence on $\gC_{\max}$.
Moreover, Corollary~\ref{cor:lower} shows that it is possible to construct games with small $\mathcal{C}(\muE,\nuE)$ but unbounded $\gC_{\max}$. In this regime, the results of \cite{freihaut2025learningequilibriadataprovably} cannot characterize the behavior of BC, whereas our results predict that learning in a non-interactive setting with unbounded $\mathcal{C}_{\max}$ is not possible.

\subsection{Main result in interactive MAIL}
This section presents our main result for the interactive MAIL setting, focusing on our new algorithm \ours{} (see Algorithm~\ref{alg:mail_warm}).  
To provide context, \citet{freihaut2025learningequilibriadataprovably} introduced MURMAIL, the first interactive MAIL algorithm with theoretical guarantees. Their analysis shows that avoiding the concentrability coefficient $\gC(\muE,\nuE)$ requires a queriable expert. While this enables effective minimization of the Nash Gap, it comes at the cost of $\mathcal{O}(\varepsilon^{-8})$ expert queries.\\

This raises a gap: when the concentrability coefficient is small and bounded, non-interactive imitation learning algorithms can outperform existing interactive methods. To close this gap, we introduce \ours~(Algorithm~\ref{alg:mail_warm}), which reduces the required number of expert queries from $\mathcal{O}(\varepsilon^{-8})$ to $\mathcal{O}(\varepsilon^{-2})$ by leveraging a reward-free warm-up phase. As a result, the sample complexity of interactive MAIL becomes comparable to the Behavior Cloning upper bound, while entirely removing dependence on the concentrability coefficient.  \\

Since each trajectory collection queries the expert $H$ times, the total number of expert queries is bounded by $\mathcal{O}(N H) = \gO\!\left(\frac{H^7 S^3 A^2 B \log(S/\delta_{\mathrm{fail}})}{\varepsilon^2}\right).$
In contrast to MAIL-BRO \citep{freihaut2025learningequilibriadataprovably}, our approach requires no additional assumptions such as access to a best-response oracle, which takes an input the policy of one player, and outputs a best responding policy for the other player. Up to problem-dependent factors in $H,S,A$, the rate matches that of Behavior Cloning but crucially avoids reliance on the concentrability coefficient, which can be unbounded and cause non-interactive methods to fail. Thus, \ours~ improves the best known guarantees by an order of $\varepsilon^{-6}$. Moreover, by adapting the construction from Theorem~\ref{thm:lower_bound}, we show that the rate $\mathcal{O}(\varepsilon^{-2})$ is optimal.  

The following theorem states our main guarantee for \ours.  

\begin{theorem}
\label{thm:main_result}
    For any $\varepsilon > 0$ and $\delta_{\mathrm{fail}} \in (0,1)$, if we execute Algorithm~\ref{alg:mail_warm} and choose the parameters according to $N = \gO(\frac{H^6S^3A^3_{\max} \log(S/\delta_{\mathrm{fail}})}{\varepsilon^2})$ and $N_0 \geq \mathcal{O}\brr{S^3A^2_{\max}H^6\iota^3_0/\varepsilon},$ we get with probability $1-\delta_{\mathrm{fail}}$ for the policies $(\widehat{\mu}, \widehat{\nu})$ that
    \[\Nashgap(\widehat{\mu}, \widehat{\nu}) \leq \mathcal{O}(\varepsilon).\]
\end{theorem}

This result establishes a sample complexity guarantee of $\gO(\varepsilon^{-2})$, which matches our lower bound and is therefore rate-optimal in the sense that the dependence on $\varepsilon$ cannot be improved. Consequently, \ours~ achieves the optimal sample complexity in the interactive setting, closing the remaining theoretical gap in interactive MAIL.  

\section{Lower bound }
We now provide intuition for our lower bound result. The key insight is that the boundedness of the concentrability coefficient $\gC(\muE,\nuE)$  is not sufficient to guarantee that a non-interactive imitation learning algorithm outputs an $\varepsilon$-Nash equilibrium. \\

The reason is that bounded $\gC(\muE,\nuE)$ only guarantees that all states that can be visited by \emph{rational} deviations have positive probability under the dataset distribution $\rho$. By rational deviation, we mean that the opponent might choose to act only according to policies which are guaranteed not to be exploited by a Nash equilibrium. However, we show that the opponent could exploit the output of a non-interactive MAIL algorithm deviating \emph{irrationaly}, in the sense that such deviation would be exploitable by a Nash profile. Bounded $\gC_{\max}$ imposes that also the states reachable by irrational deviations are covered by the sampling distribution $\rho$. \\

To illustrate this, consider the Markov game instance from Theorem~\ref{thm:lower_bound}, depicted in Figure~\ref{fig:lower}. Both agents start in state $s_1$, where agent 2 controls the transition: action $b_1$ leads to $s_2$, while action $b_2$ leads to $s_3$, regardless of agent 1’s choice. The rewards are $r(s_1) = r(s_2) = 0$, while in $s_3$ the agents play a normal-form game with Nash value of $1$. Importantly, in equilibrium, both players visit only $s_1$ and $s_2$ when acting according to a NE, which ensures that $\mathcal{C}(\muE, \nuE)$ is bounded. However, since the learner never observes equilibrium play in $s_3$, it cannot recover the correct Nash strategy for agent 1 in that state. This allows agent 2 to deviate \emph{irrationaly} by choosing $b_2$ and exploit the learner in $s_3$. \\

The lower bound is established with well-established information techniques ~(see e.g. \citep{lattimore2020bandit})~  by constructing a family of two Markov games in which the payoff structure in $s_3$ corresponds to different versions of Matching Pennies, with the unique Nash value differing by $\varepsilon$. Distinguishing these two games requires $\Omega(\varepsilon^{-2})$ samples. Since $s_3$ is visited only $\gC_{\max}$ many times under the expert equilibrium, a non-interactive learner cannot collect enough data to resolve this ambiguity, which yields the lower bound in Theorem~\ref{thm:lower_bound}. A detailed proof is provided in Appendix~\ref{appendix:lower_bound}.

\begin{figure}[h!]
    \centering
\begin{tikzpicture}[>=stealth',shorten >=1pt,auto,node distance=2.8cm]
  \node[state] (s0)  {$s_1$};
  \node[state] (s1)[right of=s0] {$s_2$};
  \node[state] (s2) [below of=s1] {$s_3$};
  \path[->] 
    (s0) edge node {$a_1 b_1$, $a_2 b_1$} (s1);
  \path[->] 
    (s0) edge node {$a_1 b_2$, $a_2 b_2$} (s2);
\end{tikzpicture}
    \caption{Markov game instance used for the Lower bound.}
    \label{fig:lower}
\end{figure}
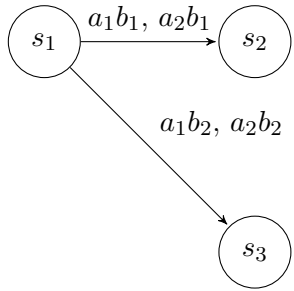

\section{Interactive MAIL algorithm}
In this section, we describe our proposed algorithm \ours~ in detail. The algorithmic pseudocode can be found in Algorithm~\ref{alg:mail_warm}. In the first phase, we loop over states and stages. For each $s, h \in \gS \times [H]$, we instantiate a reward-maximization problem in which the only nonzero, positive reward is obtained by reaching state $s$ after $h$ stages in the MDP induced by fixing the $\nu$-player to follow the Nash equilibrium policy $\nuE$. 
At line~$6$, this RL problem is solved via EULER \citep{zanette2019tighterproblemdependentregretbounds}, which outputs a sequence of $N_0$ policies satisfying a first-order regret bound.
At line~$8$, we set the policies at all stages after $h$ to be uniform. This choice is arbitrary, the effective reward instance has an horizon of length $h$ and the policy can be arbitrary set after step $h$ without changing the regret properties.
At the end, of the initial for loop we obtain a policy set $\Psi^{\nuE}$ containing $S H N_0$ policies.
\begin{algorithm}[!h]
\caption{Multi-Agent Imitation Learning with reward-free warm-up (\ours)}
\label{alg:mail_warm}
\begin{algorithmic}[1]
\State \textbf{Input:} iteration number $N_0$, $N$, queriable expert $\nuE$.
\State \textbf{Reward-free warm-up phase:}
\State set policy class $\Psi^{\nuE} \leftarrow \emptyset$, and dataset $\mathcal{D} \leftarrow \emptyset$.
\For{all $(s,h) \in \mathcal{S} \times [H]$}
    \State $r_h^{\nuE}(s', a') \leftarrow \mathbf{1}[s'=s \text{ and } h'=h]$ for all $(s',a',h') \in \mathcal{S} \times \mathcal{A} \times [H]$.
    \State $\bcc{\pi_i^{(s,h)}}^{N_0}_{i=1} \leftarrow \text{EULER}(r^{\nuE}, N_0, P^{\nuE})$.
    \State Let $\Phi^{(s,h)} \gets \bcc{\pi_i^{(s,h)}}^{N_0}_{i=1}$
    \State $\mu_{h'}(\cdot|s) \leftarrow\text{Unif}(\mathcal{A})$, $~\forall \mu \in \Phi^{(s,h)}, \forall h' \geq h $.
    \State $\Psi^{\nuE} \leftarrow \Psi^{\nuE} \cup \Phi^{(s,h)}$.
\EndFor
\For{$n = 1 \dots N$}
    \State sample policy $\mu \sim \text{Unif}(\Psi^{\nuE})$.
    \State Collect $ z_n = (s_1, a_1, b_1, \dots, s_{H+1})  \sim \mu, \nuE$.
    \State $\mathcal{D}^{\nuE} \leftarrow \mathcal{D}^{\nuE} \cup \{z_n\}$
\EndFor
\State \textbf{Repeat}: the reward-free warm-up phase fixing $\muE$ for fixed $\muE$ generating the policy set $\Psi^{\muE}$ and the dataset $\mathcal{D}^{\muE}$.
\State \textbf{Receive:} datasets $(\mathcal{D}^{\muE},\mathcal{D}^{\nuE})$.
\State \textbf{Imitation Learning}
\State  Use dataset $\mathcal{D}^{\nuE}$ to compute 
\[
\hat{\nu} = \argmin_{\nu \in \Pi_{\nu}} \sum_{s,b \in \mathcal{D^{\nuE}} } - \log \nu(b|s)
\]
\State Use dataset $\mathcal{D}^{\muE}$ to compute 
\[
\hat{\mu} = \argmin_{\mu \in \Pi_{\mu}} \sum_{s,a \in \mathcal{D^{\muE}} } - \log \mu(a|s)
\]
\State \textbf{Return} Nash estimate $(\widehat{\mu},\widehat{\nu}).$
\end{algorithmic}
\end{algorithm}
Next, we can exploit this set to construct an exploratory dataset $\mathcal{D}^{\nuE}$ of $N$ state action sequences in the Markov game. A trajectory is collected fixing the $\nu$-player to the strategy $\nuE$ while the $\mu$ player acts according to a policy sampled uniformly at random from the set $\Psi^{\nuE}$. Repeating the sampling process for choosing the policy of the player $\mu$, $N$ times, gives the desired dataset.
Switching the player that remains fixed and reapplying the same procedure, we obtain the exploratory dataset $\mathcal{D}^{\muE}$ for the single player MDP induced by fixing the $\mu$-player to the policy $\muE$.
Finally, the last step is to use the datasets $\mathcal{D}^{\muE}$  and $\mathcal{D}^{\nuE}$ in lines $19$ and $20$ to apply behavioral cloning over the state-action pairs in these datasets. \\

All in all, our new algorithm is based on the intuition of using a reward-free  routine to design an exploratory dataset over which we can prove benign bounds for behavioral cloning which are actually independent of the concentrability factor.
This is without contradiction with our lower bound since in the dataset construction phase we used interaction with the expert. \\

Compared to \emph{MURMAIL} and \emph{MAIL-BRO} by \citet{freihaut2025learningequilibriadataprovably}, our new algorithm requires neither a best response oracle nor to compute the maximum uncertainty response by solving a RL inner-loop at every iteration. In particular, the maximum uncertainty response in MURMAIL required to compute at each iteration $k\in [K]$ an RL problem with reward $\norm{\nu_k(\cdot|s) - \nuE(\cdot|s)}^2$. In comparison,  \ours ~solves only $SH$ RL problems with rewards independent of a particular policy $\nu_k$.  Since the number of RL problems no longer depends on $K$, \ours~ achieves better guarantees. \\

This intuition is formalized in the next section, which presents the key steps of the proof of Theorem~\ref{thm:main_result}.
\section{Analysis}
The goal of this section is to prove our main result stated in Theorem~\ref{thm:main_result}.
We begin by presenting general results on the reward-free algorithm of \cite{jin2020reward}, applied to the single-player MDPs that arise when one of the two players is fixed to the queriable expert policy. Then, we show how analyzing Behavior Cloning on the dataset generated during the reward-free warm-up phase can be used to effectively minimize the Nash gap.
\subsection{Phase 1: Reward-free Warm-up}

As evident from the algorithm presentation, the core concept of the analysis is the single-agent MDP induced by fixing one of the two players to the policy of the queriable expert. 

\begin{definition}[Expert Induced MDP]
\label{def:induced_MDP}
    Let $\gG$ be a Markov game and $\muE,\nuE$ the expert policies, then $\mathcal{M}^{\nuE}:= (\gS, \gA, P^{\nuE}, r^{\nuE}, H)$ is the MDP induced by the expert $\nuE$ with the transition model $P^{\nuE} = \bcc{P^{\nuE}_h}^H_{h=1}$ with $P_h^{\nuE} (s' \mid s,a):= \sum_{b\in\gB}  \nuE_h(b \mid s) P_h(s'\mid s,a,b) $ and an arbitrary reward function $r^{\nuE} = \bcc{r^{\nuE}_h}^H_{h=1},$ such that $r^{\nuE}_h(s,a) \in \{0,1\} \quad \forall (s,a) \in \gS\times \gA$. Analogously we define $\gM^{\muE}$ as the MDP induced by the expert policy $\muE$.
\end{definition}

Notice that the reward function in the induced expert MDP is not related to the true unknown reward function of the original Markov game, instead it is used for exploration purposes as shown in Algorithm~\ref{alg:mail_warm}. With this definition, we can state a fundamental result about the reward-free warm-up phase. 
In particular, Theorem~\ref{thm:reward_free_main_result} states that using the distributions $p_h^{\muE}(s,b) := (N_0 S H)^{-1}\sum_{\nu \in \Psi^{\muE}}d_h^{\muE, \nu}(s) \nu(b|s)$ and $p_h^{\nuE}(s,a) := (N_0 S H)^{-1}\sum_{\mu \in \Psi^{\nuE}}d_h^{\mu, \nuE}(s) \mu(a|s)$ to generate the dataset (that is using them as distribution $\rho$) guarantees a benign bound on the coverage over the set of $\delta$-reachable states of any possible occupancy measure in the induced MDPs. By coverage, we mean the ratio $\max_{a,h}d_h^{\mu,\nuE}(s,a)/p_h^{\nuE}(s,a)$ which at an intuitive level can be thought at the average number of times one needs to sample from  $p^{\nuE}$ before hitting a state action pair which has probability $d^{\mu,\nuE}(s,a)$ under the policy $\mu$. \\

Technically, the result is obtained applying Theorem 3.3 of \citet{jin2020reward} twice in the induced MDPs $\mathcal{M}^{\nuE}$ and  $\mathcal{M}^{\muE}$. 
\begin{theorem}
\label{thm:reward_free_main_result}
    Let $\mathcal{M}^{\nuE}$ be the induced MDP defined in Definition~\ref{def:induced_MDP} and the policy $\Psi^{\nuE}$ set generated according to Algorithm~\ref{alg:mail_warm}. Then there exists  an absolute constant $c>0$ such that for any $\varepsilon > 0 $ and $p\in(0,1),$ if we set $N_0 \geq cS^2AH^4\iota^3_0/\delta,$ where $\iota_0:= \log(SAH/p\delta),$ then with probability $1-p,$ the reward free exploration returns a sampling distribution $p^{\nuE}$ such that for all $\mu \in \Pi_{\mu}$
\begin{align}
\label{eq:reward_free_coverage}
        \forall \delta-\mathrm{significant} (s,h), \quad \max_{a,h} \frac{d^{\mu,\nuE}_h(s,a)}{p_h^{\nuE}(s,a)} \leq 2SAH,
\end{align}
    where $(s,h)$ is $\delta-\mathrm{significant},$ if the probability to reach a state $s$ at step $h$ in the induced MDP $\mathcal{M}^{\nuE}$ is lower bounded by $\delta$:
    \[\max_{\mu\in\Pi_{\mu}} d^{\mu,\nuE}_h(s) \geq \delta.\]
 Analogously, we get the same result for $\delta$-reachable states in $\mathcal{M}^{\muE} $ with sampling distribution $p^{\muE}$.
\end{theorem}
For convenience, we will denote the set of $\delta$-significant states in the MDP $\mathcal{M}^{\nuE}$ as $\gS^{\nuE}_{\delta,h}$. For the other player, $\gS^{\muE}_{\delta,h}$ is defined analogously. \\

In reward-free RL this dataset is taken to enable learning against any reward function. Here, instead we find a new interesting use case of the warm-up phase which is to use the generated datasets $\mathcal{D}^{\muE}$ and  $\mathcal{D}^{\nuE}$ for running behavioral cloning. \\

In particular, we can perform a change of measure to avoid the need of sampling states according to the occupancy measure of the best response policy which cannot be computed. Given the good coverage properties of the distributions $p^{\nuE}$ and $p^{\muE}$, the change of measure creates an increase in the number of needed samples of at most $SH A_{\max}$. Since the increase is independent on the desired accuracy $\varepsilon$ we can retain the optimal $\mathcal{O}(\varepsilon^{-2})$ rate. \\

Next, we can see how the coverage property in Theorem~\ref{thm:reward_free_main_result} enables to prove Theorem~\ref{thm:main_result}.

\subsection{Phase 2: BC over the datasets generated in Phase 1 }
In the last section, we have seen how one can construct datasets $\mathcal{D}^{\nuE}$ and $\mathcal{D}^{\muE}$ using ideas from reward-free RL.
We will use the coverage property of these datasets on top of the following exploitability decomposition. 
\begin{lemma}[\textbf{Exploitability decomposition}]
 \label{lemma:expl_decomposition}
For any policy pair $\nu,\nu'$, we define their total variation at state $s$ as $\TV(\nu,\nu')(s) = \sum_{b \in \mathcal{B}} \abs{\nu(b|s)-\nu'(b|s)}.$ It holds that
\begin{align*}
\innerprod{\initial}{V^{\mu^{\star}, \widehat{\nu} }-   V^{ \widehat{\mu}, \nu^{\star}} } \leq  2H \sum^H_{h=1} \sum_{\pi \in \bcc{\hat{\mu},\hat{\nu}}}\mathrm{Err}_h(\pi).
\end{align*}

where we have defined $\mathrm{Err}_h(\hat{\nu}):=\max_{\mu_h \in \mathrm{br}(\widehat{\nu}_h)}\expect_{ s\sim d^{\mu, \nuE}_h} \ls \TV \lp \nuE_h, \widehat{\nu}_h \rp (s) \rs $ for player 1 and additionally $\mathrm{Err}_h(\hat{\mu}):=\max_{\nu_h \in \mathrm{br}(\widehat{\mu}_h)}\expect_{  s\sim d^{\muE,\nu}_h} \ls \TV \lp \muE_h , \widehat{\nu}_h \rp (s) \rs$ for player 2.
\end{lemma}

Now, we can see that the expectation in the definition of $\mathrm{Err}(\hat{\nu})$ depends on $\mu, \nuE$, where $\mu \in \mathrm{br}(\widehat{\nu})$ is one best response to the estimated expert. Bounding this term is challenging because the reward function is unknown, so $\mathrm{br}(\widehat{\nu})$ can neither be computed nor bounded using optimistic estimators. Equivalent considerations hold for $\mathrm{Err}(\hat{\mu})$. However, we can here make use of the reward-free warm-up phase. In particular, we can use the constructed distribution $p^{\nuE}_h$  to perform the following change of measure, and rewrite $\mathrm{Err}_h(\hat{\nu})$ in the following way
\begin{align}
    \mathrm{Err}_h&(\hat{\nu}) = \underbrace{\sum_{s \in \mathcal{S}_{\delta}} \sum_{a\in \mathcal{A}} w_h^{\widehat{\nu}}(s,a) p_h^{\nuE}(s,a) \TV \lp   \nuE_h, \widehat{\nu}_h \rp (s)}_{:= \mathrm{Err}_h( \hat{\nu};\mathcal{S}^{\nuE}_{\delta,h})} +\underbrace{\sum_{s \notin \mathcal{S}_{\delta}} \sum_{a\in \mathcal{A}} \max_{\mu \in \mathrm{br}(\hat{\nu})} d_h^{\mu,\nuE}(s,a) \TV \lp   \nuE_h, \widehat{\nu}_h \rp (s)}_{{:= \mathrm{Err}_h( \hat{\nu};\bar{\gS}^{\nuE}_{\delta,h})}},
\end{align}
where we introduce the importance weight correction  $w_h^{\widehat{\nu}}(s,a):= \max_{\mu \in \mathrm{br}(\widehat{\nu})} \frac{d_h^{\mu,\nuE}(s,a)}{p_h^{\nuE}(s,a)} $ and we split the sum over the $\delta$-significant states set $\gS^{\nuE}_{\delta,h}$ from the sum over the not significant states set denoted by $\bar{\gS}^{\nuE}_{\delta,h} := \gS\setminus\mathcal{S}^{\nuE}_{\delta,h} $.
At this point, we can bound $ \mathrm{Err}_h( \hat{\nu};\mathcal{S}^{\nuE}_{\delta,h}) $ using Theorem~\ref{thm:reward_free_main_result} which implies that for any $s \in \gS^{\nuE}_{\delta,h}$ it holds that $\max_{h\in[H],a\in \A} w^{\hat{\nu}}_h(s,a) \leq 2 SAH $. In this way, we obtain
\begin{align*}
\mathrm{Err}_h(\hat{\nu};\gS^{\nuE}_{\delta,h}) \leq 2SAH \mathbb{E}_{s \sim p^{\nuE}_h}\bs{\TV \lp   \nuE_h, \widehat{\nu}_h \rp (s)}.
\end{align*}
This term is now easy to bound since $\hat{\nu}$ is computed as the minimizer of the log loss over the dataset $\mathcal{D}^{\nuE}$ which is sampled from the distribution $p^{\nuE}_h$. A standard concentration inequality for the log loss minimizer (see Lemma \ref{lemma:bc_concentration}) guarantees that with probability at least $1-\delta_{\mathrm{fail}}/2$
\[
\mathrm{Err}_h(\hat{\nu};\gS^{\nuE}_{\delta,h}) \leq 8SAH \sqrt{\frac{S B \log(4S/\delta_{\mathrm{fail}})}{N}}.
\]
Using the same analysis switching the roles of the players ensure that
\[
\mathrm{Err}_h(\hat{\mu};\gS^{\muE}_{\delta,h}) \leq 8SBH \sqrt{\frac{S A \log(4S/\delta_{\mathrm{fail}})}{N}} ~~\text{w.p.}~~1-\frac{\delta_{\mathrm{fail}}}{2}.
\]
Finally, by definition of the set $S_{\delta}$, it follows that the error term contribution coming from the non significant states can be bounded as
\begin{align*}
\mathrm{Err}_h(\hat{\mu};\bar{\gS}^{\muE}_{\delta,h})  \leq SA \delta, \quad \mathrm{Err}_h(\hat{\nu};\bar{\gS}^{\nuE}_{\delta,h})  \leq SB \delta.
\end{align*}
All in all,
we obtain that with probability at least $1 - \delta_{\mathrm{fail}}$ it holds that $\innerprod{\initial}{V^{\mu^\star, \hat{\nu}} - V^{\hat{\mu}, \nu^\star}}$ can be bounded by
\[
\mathcal{O}\brr{\sqrt{\frac{S^3 A^3_{\max} H^6 \log\brr{4S/\delta_{\mathrm{fail}}}}{N}} + S A_{\max} \delta }.
\]
Therefore, selecting $\delta = \mathcal{O}(\varepsilon/(S A_{\max} H^2))$ and $N = \mathcal{O}\brr{\frac{S^3 A^3_{\max} H^6 \log(4S/\delta_{\mathrm{fail}})}{\varepsilon^2}}$ proves Theorem~\ref{thm:main_result}.

In Appendix \ref{appendix:general_sum}
we show that our approach translates to the n-player general-sum setting smoothly without incurring a sample complexity exponential in the number of players.

\paragraph{On the rate optimality for \ours} The rate optimality of \ours~ follows from an even simplier construction compared to Figure~\ref{fig:lower}. In particular, consider a single state Markov game where the payoff matrices are given by the perturbed Matching Pennies game used in $s_3$ of Figure~\ref{fig:lower}. We can establish with the exact same proof of Theorem~\ref{thm:lower_bound} that $\Omega(\varepsilon^{-2})$ expert queries are absolutely necessary even in interactive MAIL.

\section{Numerical verifications}
In this section, we provide some numerical verifications supporting our theoretical analysis.
First, we consider the Markov game used for the lower bound (see Figure \ref{fig:lower}). Additionally, we compare BC, \emph{MURMAIL} and \emph{\ours} in a zero-sum Gridworld. The 3x3 Gridworld, visualized in Figure~\ref{fig:grid}, has $72$ states, where each state indicates the position of both players. Both agents cannot be at the same position. In the top right corner of the Gridworld, there is a goal state. If an agent reaches the goal before the other agent does, this agent receives a reward of $1,$ and the other agent gets a reward of $-1.$ Gridworld 1 and 2 refer to the same environment described above, the only difference is that in Gridworld 2 we consider observing data from a mixed Nash created taking convex combination of multiple Nash equilibria. Therefore, the dataset state coverage in Gridworld 2 is better. For these created games we run a Nash equilibrium solver, zero-sum Value Iteration \citep{pmlr-v37-perolat15}, to obtain the Nash expert policies $(\muE,\nuE)$. We repeat this for every game $10$ times across varying seeds. To plot the performance, we compute the exploitability of the estimated experts for different sample sizes of the sampled states from the obtained dataset, i.e. the expert queries. We plot the average obtained exploitability across games and trials as well as the standard deviation for the different sizes. The code used for the experiments is available at \href{https://github.com/emanuelenevali/MAIL_WARM}{https://github.com/emanuelenevali/MAIL\_WARM}.\\

In Figure~\ref{fig:experiments}, we quantify the expected behavior for BC in the hard instance used for our lower bound. We observe that BC's performance degrades as the state coverage decreases and that BC is not able to learn if $\gC_{\max}$ is infinite, even though $\gC(\muE,\nuE)$ is constant in all considered environments. This verifies that  $\gC_{\max}$ and does not  $\gC(\muE,\nuE)$ predict the performance of BC and any other non-interactive MAIL algorithm.\\

Interestingly, we observe that BC suffers from a constant Nash Gap in both described Gridworld set-ups. In picture $(b)$, the Nash Gap of BC is higher compared to picture $(c)$ as the expert data has better coverage. In contrast to BC, we observe that for MURMAIL and \ours, the exploitability decreases with the amount of expert queries and is independent of the expert coverage. However, \ours~ outperforms MURMAIL requiring significantly fewer expert queries to minimize the Nash Gap.

\begin{figure*}[h!]
    \centering
        \includegraphics[width=\linewidth]{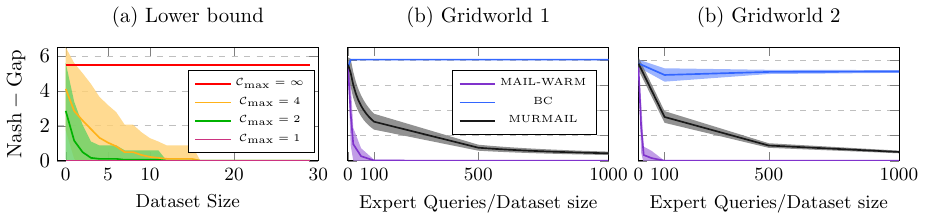}
        \caption{Exploitability of BC in the lower-bound Markov game and comparison of imitation learning algorithms in Gridworlds for a pure NE expert (Gridworld 1) and a mixed one (Gridworld 2).}
        \label{fig:experiments}
\end{figure*}

\section{Conclusion and future directions}
In this work, we resolved \textbf{Open Question 1} for the non-interactive setting and \textbf{Open Question 2} for interactive MAIL. For both cases, we established rate-optimal results, showing that the dependence on $\varepsilon$ cannot be improved.\\

Our analysis leaves several important directions open. First, while we have closed the gap in $\varepsilon$-dependence, optimal guarantees with respect to other problem parameters $S$, $A$, and $H$ remain unknown.\\

Second, our results are developed in the finite-horizon setting. Extending them to the infinite-horizon case is non-trivial, since the reward-free warm-up phase relies on regret guarantees of the EULER algorithm that are specific to finite horizons (see Section 8 in \citet{zanette2019tighterproblemdependentregretbounds}). Whether analogous results can be obtained in the infinite-horizon regime remains an open challenge, and progress in this direction could be of interest independent of MAIL.\\

Third, we provided a Markov game (Figure \ref{fig:lower}), where $\gC_{\max}$ can be computed explicitly. However, even in simple environments such as Gridworld, this calculation becomes challenging. As $\gC_{\max}$ is the key fundamental quantity describing the hardness of non-interactive MAIL, developing efficient methods for calculating it would be highly valuable. In particular this would help to clarify the practical feasibility of behavioral cloning in standard MARL benchmarks.\\

Finally, a next step is to extend our results to the non-tabular case and to provide deep MAIL algorithms building on our framework.
\bibliography{references}
\bibliographystyle{abbrvnat}

\clearpage
\appendix

\section{Related Work}
\label{appendix:related_work}
The related works can be split into into Multi-Agent Imitation Learning and reward-free Reinforcement Learning. 
\paragraph{Multi-Agent Imitation Learning}
Most of the works in Multi-Agent Imitation Learning are on the empirical side. Some works consider the cooperative setting, meaning that each agents maximizes the same reward function \citep{NEURIPS2024_2fbeed1d, 10.5555/3305381.3305587, bui2025misodicemultiagentimitationunlabeled, bui2023mimickingdominateimitationlearning}. This line of work does not necessitate any equilibrium solution framework, which is essential for our work. \citet{pmlr-v97-yu19e} or \citet{NEURIPS2018_240c945b} extend the framework of adversarial imitation learning to the multi-agent setting. In particular, \citet{NEURIPS2018_240c945b} extend GAIL \citep{Ho:2016b} to the multi-agent setting. While the authors consider Nash equilibrium experts, their theoretical results require a unique Nash equilibrium solution, which rarely holds in Markov games. \citet{pmlr-v97-yu19e}  consider inverse reinforcement learning and introduce regularization to make the Nash equilibrium, technically not a Nash equilibrium anymore, unique. In this work, we consider Nash equilibrium experts, the most common solution concept in Markov games, without any additional assumption.

In the context of mean-field games, \citet{ramponi2023imitationmeanfieldgames} were the first to study imitation learning through the lens of the Nash Gap. For general $n$-player games, the seminal work of \citet{tang2024multiagentimitationlearningvalue} established fundamental hardness results for minimizing the Nash Gap. They provided guarantees for behavioral cloning (BC) under the assumption that the Nash equilibrium policy profile generating the data visits every state with positive probability.
The closest line of work to ours is that of \citet{freihaut2025learningequilibriadataprovably}, who extended these results in several directions. First, they showed that the strict coverage assumption of \citet{tang2024multiagentimitationlearningvalue} can be dropped, and instead provided a tighter BC guarantee in terms of the all-policy deviation concentrability coefficient $\gC_{\max}$. Their analysis also established a fundamental separation between two settings. In the non-interactive setting, where the learner only has access to a fixed dataset of expert trajectories, they proved that the dependence on $\gC(\muE,\nuE)$ is unavoidable. In particular, if $\gC(\muE,\nuE)$ is infinite, then no non-interactive algorithm can succeed, even with unlimited data and known transitions. In the interactive setting, they introduced the first algorithm with sample complexity guarantees, \emph{MURMAIL}, which achieves $\tilde{\gO}(\varepsilon^{-8})$ queries, independent of $\gC(\muE,\nuE)$.
Our work sharpens these results in both regimes. In the non-interactive case, we close the theoretical gap by identifying the precise dependence on $\gC_{\max}$ and proving that behavioral cloning is rate-optimal. In the interactive case, we propose a new framework that reduces the query complexity to $\tilde{\gO}(\varepsilon^{-2})$, matching the dependence on $\varepsilon$ implied by our lower bound.
\looseness=-1

\paragraph{Reward-free Reinforcement Learning}
In their seminal work, \citet{jin2020reward} introduced the framework of reward-free reinforcement learning in the single-agent setting. The central idea is to construct a dataset without knowledge of the reward function that provides sufficient coverage of the environment so that an optimal policy can later be learned for any reward specified afterward. The framework naturally decomposes into two phases: an exploration phase, where trajectories are collected to ensure broad coverage, and a planning phase, where the collected data is used to compute an $\varepsilon$-optimal policy once a reward is revealed. Their proposed algorithm achieves this goal with sample complexity of order $\tilde\gO(\tfrac{H^5 S^2 A}{\varepsilon^2})$, and is accompanied by a nearly matching lower bound of $\Omega(\tfrac{H^2 S^2 A}{\varepsilon^2})$.

This gap has since been closed progressively. \citet{pmlr-v132-kaufmann21a} improved the upper bound to $\tilde\gO(\tfrac{H^4 S^2 A}{\varepsilon^2})$, and \citet{menard2020fastactivelearningpure} further refined it to $\tilde\gO(\tfrac{H^3 S^2 A}{\varepsilon^2})$, which matches the lower bound for non-stationary transition dynamics. The same bound has also been achieved independently by \citet{li2024minimaxoptimalrewardagnosticexplorationreinforcement}. Beyond tabular MDPs, the reward-free framework has been extended to the linear MDP setting \citep{pmlr-v162-wagenmaker22b, zhang2024optimalhorizonfreerewardfreeexploration}.
\looseness=-1

Here, we focus primarily on the exploration phase of reward-free RL. Specifically, we leverage the idea of constructing datasets for induced expert MDPs, each of which produces a state distribution that provides sufficient coverage of the relevant state space. This principle serves as the key ingredient for deriving our sample-efficient algorithm.
\looseness=-1

\section{Proof of Lower bound}
\label{appendix:lower_bound}
In this section, we provide the detailed proof of Theorem \ref{thm:lower_bound} and Corollary~\ref{cor:lower}. Both proofs require the same Markov game construction, which is again illustrated in Figure~\ref{fig:lower_appendix} to provide further intuition for the proof.

\begin{figure}[h!]
    \centering
\begin{tikzpicture}[>=stealth',shorten >=1pt,auto,node distance=2.7cm]
  \node[state] (s0)  {$s_1$};
  \node[state] (s1)[right of=s0] {$s_2$};
  \node[state] (s2) [below of=s1] {$s_3$};
  \path[->] 
    (s0) edge node {$a_1 b_1$, $a_2 b_1$} (s1);
  \path[->] 
    (s0) edge node {$a_1 b_2$, $a_2 b_2$} (s2);
\end{tikzpicture}
    \caption{Markov game instance used for the Lower bound.}
    \label{fig:lower_appendix}
\end{figure}
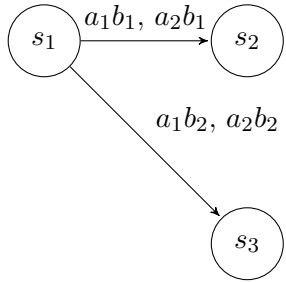

The main idea of the proof is twofold. In a first step, we show that even if $\gC(\muE,\nuE)$ is bounded, no non-interactive Multi-Agent Imitation Learning algorithm can learn an equilibrium from data as long as $\gC_{\max} = \infty.$ Then, we construct a state with two different reward functions such that the unique Nash equilibrium value for these games only differ by $\varepsilon.$ This construction allows to show, that at least an expert dataset of size $\Omega(\varepsilon^{-2})$ is required to learn an Nash equilibrium from data. Additionally, noting that the expected visits of this state are given by $\gC_{\max}$ completes the proof. 
Let us first restate both results (respectively Thm.~\ref{thm:lower_bound} and Cor.~\ref{cor:lower}).

\begin{theorem}
    Let $\hat{\mu},\hat{\nu}$ be the output of a non-interactive MAIL algorithm $\mathrm{Alg}$. Then, for any  $\mathrm{Alg}$, there exists a Markov game such that satisfying $\mathbb{E}_{\mathrm{Alg}}\bs{\innerprod{\initial}{V^{\mu^{\star} , \widehat{\nu} } -   V^{ \widehat{\mu}, \nu^{\star} }}} \leq \gO(\varepsilon)$ requires an expert dataset of size $N=\Omega(\frac{\gC_{\max}}{\varepsilon^2})$.
\end{theorem}
The same construction gives the following corollary for unbounded $\mathcal{C}_{\max}$.
\begin{corollary}
    For any non-interactive MAIL algorithm, there exists a Markov game $\gG$ with $\gC(\muE,\nuE) <\infty$ and $\gC_{\max} = \infty$ where $\mathbb{E}_{\mathrm{Alg}}\bs{\innerprod{\initial}{V^{\mu^{\star} , \widehat{\nu} } -   V^{ \widehat{\mu}, \nu^{\star} }}} \geq \frac{H-1}{60}$.
\end{corollary}
Next, we proceed with the proof of Theorem~\ref{thm:lower_bound}.
\begin{proof}
    We start the proof with the setup of the Markov game (see Figure~\ref{fig:lower_appendix}). 
Let us consider a family of two Markov games $\mathcal{H} = \bcc{\mathcal{G}_1, \mathcal{G}_2, }$ with states structure illustrated in Figure \ref{fig:lower}. The state space is $\gS = \bcc{s_1, s_2, s_3}$, the action set for the $\mu$-player is $\gA = \bcc{ a_1, a_2}$ and for the $\nu$ player is $\gB = \bcc{ b_1, b_2}$. The initial joint state is $s_1$, then the next state is $s_2$ when the $\nu$ player plays the action $b_1$ and $s_3$ if $b_2$ is played instead. That is, the next state is fully decided by the $\nu$ player no matter which is the action sampled by the $\mu$ player.
The reward function is given by $r(s_1,a,b) = 0$ and $r(s_2,a,b) = 0$ for all Markov games in the family $\mathcal{H}$. On the contrary the reward function in the state $s_3$ differs among the members of the family $\mathcal{H}$. For any game $\mathcal{G} \in \mathcal{H}$, there exist a scalar parameter $\Delta_{\mathcal{G}} > 0$
which parameterize the payoff matrix as follows 
\begin{equation*}
r_{\mathcal{G}}(s_3,a,b) = \mathbf{e}_a^{T}
\underbrace{\begin{pmatrix}
1 + \Delta_{\mathcal{G}} & -1 \\
-1 & 1 \\
\end{pmatrix}}_{:= R_{\mathcal{G}}} \mathbf{e}_b,
\end{equation*}
where $\mathbf{e}_a$ and $\mathbf{e}_b$ are indicator vectors for the actions $a$ and $b$ respectively.

Note that for $\Delta_\gG = 0$ this game is known as \emph{Matching Pennies} and it is well known that the unique Nash equilibrium corresponds to a uniform policy for both players. In particular, we next consider $\mathcal{H} = \bcc{\mathcal{G}_1, \mathcal{G}_2}$ with $\Delta_{\mathcal{G}_1} = 2 \varepsilon$ and $\Delta_{\mathcal{G}_2} = \varepsilon.$ Therefore, both members in the family $\gH$ can be seen as $\Delta$-perturbations of \emph{Matching Pennies}. For a mixed Nash equilibrium it holds true that a player must be indifferent between both actions (see e.g.\citet[Section 3.2.2]{osborne1994course}). Therefore, for the two arbitrary strategies $\mu=(p,1-p)$ and $\nu=(q,1-q)$, we get
\begin{align}
    (1+\Delta_\gG) p - (1-p) = -p + 1 -p \Leftrightarrow p = \frac1{2+\frac{\Delta_{\mathcal{G}}}{2}}.
\end{align}
Therefore, the Nash strategy is given by
\begin{align}
\mu_{\mathrm{Nash},\mathcal{G}} = \lp\frac1{2+\frac{\Delta_{\mathcal{G}}}2}, 1 - \frac1{2+\frac{\Delta_{\mathcal{G}}}2}\rp^{\top} = \lp\frac1{2} -\frac{\Delta_{\mathcal{G}}}{8 + 2 \Delta_{\mathcal{G} }}, 
\frac1{2} +\frac{\Delta_{\mathcal{G}}}{8 + 2 \Delta_{\mathcal{G} }}\rp^{\top}  .
\end{align}
Note that for $\Delta_{\mathcal{G}}=0$ we recover the Nash equilibrium strategies of the standard version of \emph{Matching Pennies}. Because of the symmetry of the game $q=\frac1{2+\frac{\Delta_{\mathcal{G}}}2}$ and can denote analogously $\nu_{\mathrm{Nash},\mathcal{G}}$. Having defined the Nash strategies, we can compute the unique value of the game for player one as 
\begin{align}
\mu_{\mathrm{Nash},\mathcal{G}}^{\top} R_{\mathcal{G}} \nu_{\mathrm{Nash},\mathcal{G}}=\frac{\Delta_{\mathcal{G}}/2}{2+\Delta_{\mathcal{G}}/2} := v_\gG
\end{align}

Next, we compute the Nash Gap as a function of $\Delta_{\mathcal{G}} \in \mathbb{R}, p,q \in [0,1]$  when they play according to policies 
\[
\hat{\mu}(s_3) = [p, 1- p] \quad \hat{\nu}(s_3) = [q, 1 - q] 
\]

In particular, we have that
\begin{align}
\mathrm{Exploitability}^{\nu}_{\mathcal{G}}(\hat{\nu}) &:= \max_{\mu} \mu(s_3)^{\top} R_{\mathcal{G}} \hat{\nu}(s_3) - v_\gG= \max_{p \in [0,1]} [p, 1 - p] \begin{pmatrix}
(2+\Delta_{\mathcal{G}})q -1\\
1-2q \\
\end{pmatrix} -  v_\gG \\ &= \max \bcc{ (2+ \Delta_{\mathcal{G}})q -1, 1-2q  } - v_\gG,   \\
\mathrm{Exploitability}^{\mu}_{\mathcal{G}}(\hat{\mu}) &:= v_\gG - \min_{\nu} \hat{\mu}(s_3)^{\top} R_{\mathcal{G}} \nu(s_3) = v_\gG - \min_{q \in [0,1]} [q, 1 - q] \begin{pmatrix}
(2+ \Delta_{\mathcal{G}})p -1\\
1-2p \\
\end{pmatrix} \\ &= v_\gG - \min \bcc{ (2+ \Delta_{\mathcal{G}})p - 1, 1-2p  } ,
\end{align}
Combining both exploitabilities, we can derive 
\begin{align}
\Nashgap_{\mathcal{G}}(\hat{\mu},\hat{\nu}) = \max \bcc{ (2+ \Delta_{\mathcal{G}})q -1, 1-2q  } - \min \bcc{ (2+ \Delta_{\mathcal{G}})p - 1, 1-2p  }. 
\end{align}

Observe that for the Nash strategies  $\mu_{\mathrm{Nash},\mathcal{G}}$ and $\nu_{\mathrm{Nash},\mathcal{G}}$ we have that the exploitability equals $0$.

Therefore, if both players play according to the Nash equilibrium policies in the state $s_3$ the first player gain reward $v_\gG = \frac{\Delta_{\mathcal{G}}}{4 + \Delta_{\mathcal{G}}} >0$ which is larger than what can be gained in the state $s_2$. Vice-versa, the $\nu$-player can get at most $-v_\gG = - \frac{\Delta_{\mathcal{G}}}{4 + \Delta_{\mathcal{G}}} < 0$ playing against the Nash profile for the $\mu$ player. It follows that the $\nu$-player can always get a higher reward in the state $s_2$. 
Recalling the dynamics in the state $s_1$, we notice that the $\nu$-player can always ensure that the next visited state is $s_2$ by playing the action $b_1$.
It follows that in the state $s_1$, the Nash equilibrium policies are $\muE(s_1)$ arbitrary and $\nuE(b_1 |s_1) = 1$ and $\nuE(b_2 |s_1) = 0$.
Therefore, for any arbitrary choice of $\muE$, we have that $\nuE$ is the unique best response.
Therefore, we have that for any choice of $\muE$ it holds that the occupancy measure equals
\[
 d^{\muE,\nuE}_0(s_1) = 1 \quad d^{\muE,\nuE}_1(s_2) = 1 \quad d^{\muE,\nuE}_1(s_3) = 0.
\]
Therefore, for any offline state sampling distribution $\rho = \bcc{\rho_h}^1_{h=0}$ with $\rho_0(s_1) =1$, we have that the single policy deviation coefficient for the $\mu$-player is given by 
\[
\mathcal{C}(\muE,\nuE)= \max\bcc{\frac{d^{\muE,\nuE}_0(s_1)}{\rho_0(s_1)}, \frac{d^{\muE,\nuE}_1(s_2)}{\rho_1(s_1)}} = \max\bcc{\frac{1}{\rho_0(s_1)}, \frac{1}{\rho_1(s_2)}} = \rho_1^{-1}(s_2),
\]
where the equality follows from the fact that $\rho_0(s_1) =1$.
Therefore, to ensure a bounded value of $\mathcal{C}(\muE, \nuE)$ it is enough to choose $\rho$ to have support including the states $s_1, s_2$ but not necessary $s_3$. However, let us now assume that the non-interactive MAIL algorithm at hand outputs a policy $\hat{\mu}$, no matter how this policy is produced. The $\nu$-player can choose to play according the policy $\nu_{\mathrm{exploit}}$ such that  $\nu_{\mathrm{exploit}}(b_2|s_1) = 1$ to ensure that $ d^{\muE,\nu_{\mathrm{exploit}}}_1(s_3) = 1$.
In words, the $\nu$-player can now choose an action outside the support of the Nash equilibrium to ensure that the next state is $s_3$. This would be irrational if $\hat{\mu}$ would coincide with the Nash profile in the state $s_3$ because in such situation the $\mu$-player could gain reward $v_\gG > 0$ while the $\nu$ player would get $-v_\gG$ at most.
However, if $\rho_1(s_3) = 0$ the state is never visited and there exists at least one game in $\mathcal{H}$ where its sub-optimality is constant. On the one hand, $\mathrm{C}(\muE,\nuE)$ is finite in this case.  
On the other hand, we have that the all policy deviations concentrability coefficient for the $\mu$-player is given for any arbitrary $\muE$ by
\begin{align*}
 \mathcal{C}_{\max} = \max_{\nu \in {\nuE, \nu_{\mathrm{exploit}}}} \bcc{ \frac{d^{\muE, \nu}_1(s_2)}{\rho_1(s_2)}, \frac{d^{\muE, \nu}_1(s_3)}{\rho_1(s_3)} } &= \max \bcc{\frac{d^{\muE, \nuE}_1(s_2)}{\rho_1(s_2)}, \frac{d^{\muE, \nu_{\mathrm{exploit}}}_1(s_3)}{\rho_1(s_3)} } \\
 &= \max \bcc{\frac{1}{\rho_1(s_2)}, \frac{1}{\rho_1(s_3)}} = \frac{1}{\rho_1(s_3)},
\end{align*}
where the last equality follows assuming that $\rho_1(s_3) \leq \rho_1(s_2)$.
Therefore, we showed that learning an equilibrium from data is not possible if $\max_{\mu,\nu} C(\mu,\nu)$ is unbounded.

With the same construction, we can quantify a finite time statistical rate, bounding the number of times expert actions should be seen in $s_3$ in order to learn an $\varepsilon$-approximate Nash equilibrium.

Now, let us consider the defined $\hat{\mu}(s_3) = [p, 1-p],$ parametrized as the following function of $\alpha \in [0,1]$,
\[
p = \frac{1}{2} - \brr{\frac{2 \alpha  \varepsilon}{8 + 4 \varepsilon} + \frac{(1-\alpha) \varepsilon}{8 + 2\varepsilon}},  %
\]
and $q$ as a function of $\beta\in[0,1]$ as follows,
\[
q = \frac{1}{2} - \brr{\frac{2 \beta \varepsilon}{8 + 4 \varepsilon} + \frac{(1-\beta) \varepsilon}{8 + 2\varepsilon}}.
\]
Therefore, setting $\alpha=\beta=1$, we have that $\hat{\mu},\hat{\nu}$ are equilibrium policies in $\mathcal{G}_1$. Vice-versa, setting $\alpha=\beta=0$ we have that $\hat{\mu},\hat{\nu}$ equals $\mu_{\mathcal{G}_2}, \nu_{\mathcal{G}_2}$.
Indeed, we can interpret $\alpha,\beta$ has the probability of choosing either the equilibrium profile in $\mathcal{G}_1$ or the one in $\mathcal{G}_2$.
We now proceed proving a lower bound only for the non-interactive MAIL algorithm which outputs policies $\hat{\mu},\hat{\nu}$ parameterized by values of $\alpha,\beta\in [0,1]$. In Lemma~\ref{lemma:strategy_class} we prove that it is enough to consider this restricted class of policies. The intuition is that, an algorithm that only considers these policies has an advantage compared to any other algorithm for this lower bound, as the considered policies interpolate between the Nash equilibria of the two games. In particular, Lemma~\ref{lemma:strategy_class} shows that the worst case expected exploitability can only increase for values $\alpha,\beta \notin [0,1]$.

To proceed, we write the Nash Gap in the game $\mathcal{G}_1$ as a function of $\alpha,\beta$. In particular, we have that
\begin{align}\label{eq:exploitability}
\max \bcc{ (2+ \Delta_{\mathcal{G}_1})q -1, 1-2q  } - v_{\mathcal{G}_1} = \max \bcc{ \frac{2\varepsilon(1+\varepsilon)(1-\beta)}{(4 + \varepsilon)(2 + \varepsilon)},  - \frac{2\varepsilon(1 - \beta)}{(2 + \varepsilon)(4 + \varepsilon)}}.
\end{align}
Next, we provide the detailed calculation how this was derived. As a first step, we consider the first term of the maximum expression. Let us start with the definition of $\Delta_{\gG_1}$ and $q$ to get
\begin{align}
    (2+ \Delta_{\mathcal{G}_1})q -1 - v_{\mathcal{G}_1}&= 2(1+\varepsilon) \brr{\frac{1}{2} - \brr{\frac{2 \beta \varepsilon}{8 + 4 \varepsilon} + \frac{(1-\beta) \varepsilon}{8 + 2\varepsilon}}} -1 - \frac{2\varepsilon}{4 + 2\varepsilon}. \\
\end{align}
Next, refactoring the first part of the equation and summarizing gives 
\begin{align}
    &2(1+\varepsilon) \brr{\frac{1}{2} - \brr{\frac{2 \beta \varepsilon}{8 + 4 \varepsilon} + \frac{(1-\beta) \varepsilon}{8 + 2\varepsilon}}} -1 - \frac{2\varepsilon}{4 + 2\varepsilon} \\
&= (1+\varepsilon) - 2(1+\varepsilon)\brr{\frac{ \beta \varepsilon}{4 + 2 \varepsilon} + \frac{(1-\beta) \varepsilon}{8 + 2\varepsilon}} -1 - \frac{2\varepsilon}{4 + 2\varepsilon} \\
&= \varepsilon - 2(1+\varepsilon)\brr{\frac{ \beta \varepsilon}{4 + 2 \varepsilon} + \frac{(1-\beta) \varepsilon}{8 + 2\varepsilon}} - \frac{2\varepsilon}{4 + 2\varepsilon}.
\end{align}
Next, by simplifying the expression and bringing all parts on the same denominator, we get 
\begin{align}
    &\varepsilon - 2(1+\varepsilon)\brr{\frac{ \beta \varepsilon}{4 + 2 \varepsilon} + \frac{(1-\beta) \varepsilon}{8 + 2\varepsilon}} - \frac{2\varepsilon}{4 + 2\varepsilon} \\
&= \varepsilon - \frac{\varepsilon(1+\varepsilon)}{4 + \varepsilon} - \frac{2\varepsilon}{4 + 2\varepsilon} + \beta (1 + \varepsilon) \brr{\frac{\varepsilon}{4 + \varepsilon} - \frac{\varepsilon}{2 + \varepsilon} }  \\
&= \varepsilon - \frac{\varepsilon(1+\varepsilon)}{4 + \varepsilon} - \frac{\varepsilon}{2 + \varepsilon} + \beta (1 + \varepsilon) \left( \frac{\varepsilon(2+\varepsilon) - \varepsilon(4+\varepsilon)}{(4 + \varepsilon)(2 + \varepsilon)} \right) \\
&= \left( \varepsilon - \frac{\varepsilon(1+\varepsilon)}{4 + \varepsilon} - \frac{\varepsilon}{2 + \varepsilon} \right) + \beta (1 + \varepsilon) \left( \frac{2\varepsilon + \varepsilon^2 - 4\varepsilon - \varepsilon^2}{(4 + \varepsilon)(2 + \varepsilon)} \right) \\
&= \left( \frac{\varepsilon(4+\varepsilon)(2+\varepsilon) - \varepsilon(1+\varepsilon)(2+\varepsilon) - \varepsilon(4+\varepsilon)}{(4 + \varepsilon)(2 + \varepsilon)} \right) + \beta (1 + \varepsilon) \left( \frac{-2\varepsilon}{(4 + \varepsilon)(2 + \varepsilon)} \right).
\end{align}
As a last step, we simplify the numerator and get
\begin{align}
    &\left( \frac{\varepsilon(4+\varepsilon)(2+\varepsilon) - \varepsilon(1+\varepsilon)(2+\varepsilon) - \varepsilon(4+\varepsilon)}{(4 + \varepsilon)(2 + \varepsilon)} \right) + \beta (1 + \varepsilon) \left( \frac{-2\varepsilon}{(4 + \varepsilon)(2 + \varepsilon)} \right) \\
&= \frac{\varepsilon \left[ (4+\varepsilon)(2+\varepsilon) - (1+\varepsilon)(2+\varepsilon) - (4+\varepsilon) \right]}{(4 + \varepsilon)(2 + \varepsilon)} - \frac{2\beta\varepsilon(1 + \varepsilon)}{(4 + \varepsilon)(2 + \varepsilon)} \\
&= \frac{\varepsilon \left[ (8+6\varepsilon+\varepsilon^2) - (2+3\varepsilon+\varepsilon^2) - (4+\varepsilon) \right] - 2\beta\varepsilon(1 + \varepsilon)}{(4 + \varepsilon)(2 + \varepsilon)} \\
&= \frac{\varepsilon [2 + 2\varepsilon] - 2\beta\varepsilon(1 + \varepsilon)}{(4 + \varepsilon)(2 + \varepsilon)} \\
&= \frac{2\varepsilon(1+\varepsilon)(1-\beta)}{(4 + \varepsilon)(2 + \varepsilon)}.
\end{align}

For the second part of the maximum, we again start with the definition of $q$ and the Nash value of the first game $v_{\mathcal{G}_1}$ and receive
\begin{align}
    1 - 2q - v_{\mathcal{G}_1} &= 1 - 1 + 2\brr{\frac{2 \beta \varepsilon}{8 + 4 \varepsilon} + \frac{(1-\beta) \varepsilon}{8 + 2\varepsilon}} - \frac{2\varepsilon}{4 + 2 \varepsilon}.
\end{align}
Simplifying now gives 
\begin{align}
     &1 - 1 + 2\brr{\frac{2 \beta \varepsilon}{8 + 4 \varepsilon} + \frac{(1-\beta) \varepsilon}{8 + 2\varepsilon}} - \frac{2\varepsilon}{4 + 2 \varepsilon} \\
     &= \brr{\frac{ \beta \varepsilon}{2 +  \varepsilon} + \frac{(1-\beta) \varepsilon}{4 + \varepsilon}} - \frac{\varepsilon}{2 +  \varepsilon} \\
&= \left( \frac{ \beta \varepsilon}{2 +  \varepsilon} - \frac{\varepsilon}{2 +  \varepsilon} \right) + \frac{(1-\beta) \varepsilon}{4 + \varepsilon} \\
&= \frac{\beta \varepsilon - \varepsilon}{2 + \varepsilon} + \frac{(1-\beta) \varepsilon}{4 + \varepsilon} \\
&= - \frac{\varepsilon(1 - \beta)}{2 + \varepsilon} + \frac{\varepsilon(1-\beta)}{4 + \varepsilon}.
\end{align}
As a last step, we again simplify and bring both terms on the same denominator to receive
\begin{align}
    &- \frac{\varepsilon(1 - \beta)}{2 + \varepsilon} + \frac{\varepsilon(1-\beta)}{4 + \varepsilon} \\
&= \varepsilon(1 - \beta) \left( - \frac{1}{2 + \varepsilon} + \frac{1}{4 + \varepsilon} \right) \\
&= \varepsilon(1 - \beta) \left( \frac{-(4 + \varepsilon) + (2 + \varepsilon)}{(2 + \varepsilon)(4 + \varepsilon)} \right) \\
&= \varepsilon(1 - \beta) \left( \frac{-2}{(2 + \varepsilon)(4 + \varepsilon)} \right) \\
&= - \frac{2\varepsilon(1 - \beta)}{(2 + \varepsilon)(4 + \varepsilon)}.
\end{align}
Putting both final expressions together gives \eqref{eq:exploitability}.

Similarly, in the environment $\mathcal{G}_2$, we can compute that
\begin{align}\label{eq:exploitability2}
\max \bcc{ (2+ \Delta_{\mathcal{G}_2})q -1, 1-2q  } - v_{\mathcal{G}_2} = \max \bcc{ -\frac{\beta \varepsilon}{4 + \varepsilon}, \frac{2\varepsilon\beta}{(2+\varepsilon)(4 + \varepsilon)}}.
\end{align}
For completeness, we also provide the detailed calculation next. First, we plug in the Nash value of the second game $v_{\mathcal{G}_2}$ as well as the considered strategy $q$ and the definition of $\Delta_{\mathcal{G}_2}$ and receive
\begin{align}
    2+ \Delta_{\mathcal{G}_2}q -1 - v_{\mathcal{G}_2} &= (2+\varepsilon) \brr{\frac{1}{2} - \brr{\frac{2 \beta \varepsilon}{8 + 4 \varepsilon} + \frac{(1-\beta) \varepsilon}{8 + 2\varepsilon}}} -1 - \frac{\varepsilon}{4 + \varepsilon} .
\end{align}
Next, we bring everything on the same denominator $2(4+\epsilon)$ and get
\begin{align}
    & (2+\varepsilon) \brr{\frac{1}{2} - \brr{\frac{2 \beta \varepsilon}{8 + 4 \varepsilon} + \frac{(1-\beta) \varepsilon}{8 + 2\varepsilon}}} -1 - \frac{\varepsilon}{4 + \varepsilon} \\
&= \frac{\varepsilon}{2} - \frac{\beta \varepsilon}{2} - \frac{\varepsilon(1-\beta)(2+\varepsilon)}{2(4 + \varepsilon)} - \frac{\varepsilon}{4 + \varepsilon} \\
&= \frac{\varepsilon(4+\varepsilon) - \beta\varepsilon(4+\varepsilon) - \varepsilon(1-\beta)(2+\varepsilon) - 2\varepsilon}{2(4+\varepsilon)}
\end{align}
In the last step, we simplify the numerator, giving 
\begin{align}
    &\frac{\varepsilon(4+\varepsilon) - \beta\varepsilon(4+\varepsilon) - \varepsilon(1-\beta)(2+\varepsilon) - 2\varepsilon}{2(4+\varepsilon)} \\
    &= \frac{(4\varepsilon + \varepsilon^2) - (4\beta\varepsilon + \beta\varepsilon^2) - (\varepsilon(2+\varepsilon-2\beta-\beta\varepsilon)) - 2\varepsilon}{2(4+\varepsilon)} \\
&= \frac{4\varepsilon + \varepsilon^2 - 4\beta\varepsilon - \beta\varepsilon^2 - 2\varepsilon - \varepsilon^2 + 2\beta\varepsilon + \beta\varepsilon^2 - 2\varepsilon}{2(4+\varepsilon)}\\
&= - \frac{\beta\varepsilon}{4+\varepsilon}.
\end{align}
For the second part of the maximum expression, we again plug in the definitions of $v_{\mathcal{G}_2}$ and $q$ to get 
\begin{align}
    1 - 2q - v_{\mathcal{G}_2} &= 1 - 1 + 2\brr{\frac{2 \beta \varepsilon}{8 + 4 \varepsilon} + \frac{(1-\beta) \varepsilon}{8 + 2\varepsilon}} - \frac{\varepsilon}{4 + \varepsilon} \\
\end{align}
Simplifying this, we directly get
\begin{align}
    &1 - 1 + 2\brr{\frac{2 \beta \varepsilon}{8 + 4 \varepsilon} + \frac{(1-\beta) \varepsilon}{8 + 2\varepsilon}} - \frac{\varepsilon}{4 + \varepsilon} \\
&= \brr{\frac{ \beta \varepsilon}{2 +  \varepsilon} - \frac{\beta\varepsilon}{4 + \varepsilon}} \\
&= \frac{2\varepsilon\beta}{(2 + \varepsilon)(4 + \varepsilon)}.
\end{align}
Combining both derived expressions, we receive \eqref{eq:exploitability2}.

Let us now consider that $\hat{\nu}$ is the output of a non-interactive Multi-Agent Imitation Learning algorithm $\mathrm{Alg}$ (we denote this with $\hat{\nu} = \mathrm{Alg}(\mathcal{D}_{\gG})$ ) which takes as input the pre-collected dataset sampled from the Nash profile which we denote $\mathcal{D}_{\mathcal{G}} = \bcc{B^i_{\mathcal{G}}}^N_{i=1}$ where for each $i \in [N]$, $B^i_{\mathcal{G}} \sim \nu_{\mathrm{Nash},\mathcal{G}}$.
We now have
\begin{align*}
\max_{\mathcal{G}\in\mathcal{H}} \mathbb{E}\bs{\mathrm{Exploitability}^{\nu}_{\mathcal{G}}(\mathrm{Alg}(\mathcal{D}_{\mathcal{G}}))} &\geq \frac{1}{2} \sum_{\mathcal{G}\in\mathcal{H}} \mathbb{E}\bs{\mathrm{Exploitability}^{\nu}_{\mathcal{G}}(\mathrm{Alg}(\mathcal{D}_{\mathcal{G}}))} \\
    &= \frac{1}{2} \left(\frac{2\varepsilon(1+\varepsilon)\mathbb{P}_{\mathcal{G}_1}(\mathrm{Alg}(\mathcal{D}_{\mathcal{G}_1})=\nu_{\mathrm{Nash},\mathcal{G}_2})}{(4 + \varepsilon)(2 + \varepsilon)} \right.\\
    &\left.\qquad\quad+ \frac{2\varepsilon\mathbb{P}_{\mathcal{G}_2}(\mathrm{Alg}(\mathcal{D}_{\mathcal{G}_2})=\nu_{\mathrm{Nash},\mathcal{G}_1})}{(2 + \varepsilon)(4 + \varepsilon)}\right) \\
    & \geq 
    \frac{\varepsilon}{15} \brr{\mathbb{P}_{\mathcal{G}_1}(\mathrm{Alg}(\mathcal{D}_{\mathcal{G}_1})=\nu_{\mathrm{Nash},\mathcal{G}_2})+ \mathbb{P}_{\mathcal{G}_2}(\mathrm{Alg}(\mathcal{D}_{\mathcal{G}_2})=\nu_{\mathrm{Nash},\mathcal{G}_1})} \\
    & = 
    \frac{\varepsilon}{15} \brr{\mathbb{P}_{\mathcal{G}_1}(\mathrm{Alg}(\mathcal{D}_{\mathcal{G}_1})=\nu_{\mathrm{Nash},\mathcal{G}_2})+ \mathbb{P}_{\mathcal{G}_2}(\mathrm{Alg}(\mathcal{D}_{\mathcal{G}_2})\neq\nu_{\mathrm{Nash},\mathcal{G}_2})} \\
    &\geq \frac{\varepsilon}{30} \exp\lp-\mathrm{KL}\lp\mathbb{P}_{\gG_1},\mathbb{P}_{\gG_2}\rp\rp \\
    & \geq \frac{\varepsilon}{30} \exp\lp-N \mathrm{KL}\lp
\nu_{\mathrm{Nash},\mathcal{G}_1}, \nu_{\mathrm{Nash},\mathcal{G}_2}
    \rp\rp,
\end{align*}
where we used $\beta \in [0,1]$ in the first equality and $\varepsilon \leq 1$ by assumption in the second inequality.
In the second last inequality, we used the Bretagnolle-Huber inequality \citep{bretagnolle1979estimation} which gives that for any distributions $P$ and $Q$ and event $A$ and its complementary $A^C$ it holds that
\[
P(A) + Q(A^C) \geq \frac{1}{2} \exp\lp-\mathrm{KL}\lp P, Q\rp\rp.
\]
In the last inequality, we used that via the chain rule for the KL divergence and the identically independent sampling of the datasets $\mathcal{D}_{\mathcal{G}_1}$,$\mathcal{D}_{\mathcal{G}_2}$, we can rewrite $\mathrm{KL}\lp\mathbb{P}_{\gG_1},\mathbb{P}_{\gG_2}\rp$ as $N \mathrm{KL}\lp
\nu_{\mathrm{Nash},\mathcal{G}_1}, \nu_{\mathrm{Nash},\mathcal{G}_2}\rp$.

Next, note that this can be seen as a Bernoulli random variable indicating whether we are in $\gG_1$ or $\gG_2$. 
At this point, we can treat $\nu_{\mathrm{Nash},\mathcal{G}_1}, \nu_{\mathrm{Nash},\mathcal{G}_2}$ as two Bernoulli random variables with mean $\frac{1}{2} - \frac{2 \alpha \varepsilon}{8 + 4 \varepsilon}$ and $\frac{1}{2} - \frac{ \alpha \varepsilon}{8 + 2 \varepsilon}$. Therefore,
defining $\mathrm{kl}: [0,1]^2\rightarrow \mathbb{R}$ as $\mathrm{kl}(r,s) = r \log(r/s) + (1-r) \log((1-r)/(1-s))$, recalling that the $\chi^2$ divergence between Bernoulli random variables with mean $r,s$ is given by $\chi^2(r,s) = \frac{(r -s)^2}{s(1-s)}$ and, finally, by the fact that the $\chi^2$-divergence upper bounds the $KL$-divergence we have that.
\begin{align*}
N \mathrm{KL}\lp
\nu_{\mathrm{Nash},\mathcal{G}_1}, \nu_{\mathrm{Nash},\mathcal{G}_2}
    \rp & = N \mathrm{kl}\brr{\frac{1}{2} - \frac{2 \alpha \varepsilon}{8 + 4 \varepsilon}, \frac{1}{2} - \frac{ \alpha \varepsilon}{8 + 2 \varepsilon}} \\
    &\leq N \chi^2\brr{\frac{1}{2} - \frac{2 \alpha \varepsilon}{8 + 4 \varepsilon}, \frac{1}{2} - \frac{ \alpha \varepsilon}{8 + 2 \varepsilon}} \\
    &=N\frac{\varepsilon^2(\varepsilon-4)^2}{(8+2\varepsilon)(4+\varepsilon)(4+2\varepsilon)}
\end{align*}

Next, let us consider small $\varepsilon\in (0,1).$ Then, we have
$9 \leq (\varepsilon-4)^2 \leq 4^2 = 16$. For the denominator, we get that 
\[ 8\leq (8+2\varepsilon) \leq 10, \quad 4 \leq (4+\varepsilon) \leq 5,\quad  4 \leq 4+2\varepsilon \leq 6.\]
Putting this together gives 
\[128 \leq (8+2\varepsilon)(4+\varepsilon)(4+2\varepsilon) \leq 300\]
Combining these we can bound the $\chi^2$ distance between
\begin{align}
\label{eq:bound_chi2}
    \frac{9 \varepsilon^2}{300} \leq \chi^2\lp\nu_{\mathrm{Nash},\mathcal{G}_1}, \nu_{\mathrm{Nash},\mathcal{G}_2} \rp \leq \frac{16 \varepsilon^2}{128}.
\end{align}

Plugging this into the expected exploitability gives
\begin{align}
    \max_{\mathcal{G}\in\mathcal{H}} \mathbb{E}\bs{\mathrm{Exploitability}^{\nu}_{\mathcal{G}}(\mathrm{Alg}(\mathcal{D}_{\mathcal{G}}))} &\geq \frac{\varepsilon}{30} \exp\lp-N \mathrm{KL}\lp
\nu_{\mathrm{Nash},\mathcal{G}_1}, \nu_{\mathrm{Nash},\mathcal{G}_2}
    \rp\rp \\
    & \geq \frac{\varepsilon}{30} \exp\lp- N \chi^2\brr{\frac{1}{2} - \frac{2 \alpha \varepsilon}{8 + 4 \varepsilon}, \frac{1}{2} - \frac{ \alpha \varepsilon}{8 + 2 \varepsilon}}\rp \\
    & \geq \frac{\varepsilon}{30} \exp \lp-N \frac18\varepsilon^2 \rp.
\end{align}
To complete this step we need to set the number of samples $N$ to achieve a Nash Gap of $\gO(\varepsilon).$ It follows that $N = \Omega(\frac{1}{\varepsilon^2}).$ Therefore, it requires $N=\Omega(\frac{1}{\varepsilon^2})$ to learn a $\gO(\varepsilon)$ Nash equilibrium in state $s_3.$

We remind ourselves that the expected number of times any non-interactive algorithm visits state $s_3$ is given by $\frac1{\rho_1(s_3)} = \gC_{\max}.$ Combining this with the previous step we receive
\begin{align}
    \max_{\mathcal{G}\in\mathcal{H}} \mathbb{E}\bs{\mathrm{Exploitability}_{\mathcal{G}}(\mathrm{Alg}(\mathcal{D}_{\mathcal{G}}))} & \geq \gC_{\max} \max_{\mathcal{G}\in\mathcal{H}} \mathbb{E}\bs{\mathrm{Exploitability}_{\mathcal{G}}(\mathrm{Alg}(\mathcal{D}_{\mathcal{G}}))(s_3)} \\
    &\geq \frac{\gC_{\max}\varepsilon}{30} \exp\lp-N \mathrm{KL}\lp
\nu_{\mathrm{Nash},\mathcal{G}_1}, \nu_{\mathrm{Nash},\mathcal{G}_2}
    \rp\rp \\
    & \geq \frac{\gC_{\max}\varepsilon}{30} \exp\lp- N \chi^2\brr{\frac{1}{2} - \frac{2 \alpha \varepsilon}{8 + 4 \varepsilon}, \frac{1}{2} - \frac{ \alpha \varepsilon}{8 + 2 \varepsilon}}\rp \\
    & \geq \frac{\gC_{\max}\varepsilon}{30} \exp \lp-\frac{N}8\varepsilon^2 \rp.
\end{align}

Therefore, for any non-interactive $\mathrm{Alg}$ it requires an expert dataset of size $N=\Omega(\frac{\gC_{\max}}{\varepsilon^2})$ to learn a $\gO(\varepsilon)$ Nash equilibrium from data.

To complete the proof it remains to show that the considered policy class is enough. This follows directly from Lemma \ref{lemma:strategy_class}. This completes the proof of Theorem~\ref{thm:lower_bound}.
\end{proof}

Next, we will provide the result, that all policies outside of the considered policy classes in the derived proof suffer from a higher worst case exploitability. In particular, we will show that the minimizer of the exploitability across the two games lies within the considered policy class. 

\begin{lemma}
    \label{lemma:strategy_class}
Let a parametrized Normal Form Game be given by 
\begin{align}
    \begin{pmatrix}
1 + \Delta_{\mathcal{G}} & -1 \\
-1 & 1 \\
\end{pmatrix}:= R_{\mathcal{G}} 
\end{align}
for $\Delta_{\mathcal{G}} \in \{\varepsilon,2\varepsilon\}.$
Additionally, let the following strategies be given
\begin{align}
&p_\alpha = \frac{1}{2} - \brr{\frac{2 \alpha  \varepsilon}{8 + 4 \varepsilon} + \frac{(1-\alpha) \varepsilon}{8 + 2\varepsilon}} \\
&q_\beta=\frac{1}{2} - \brr{\frac{2 \beta \varepsilon}{8 + 4 \varepsilon} + \frac{(1-\beta) \varepsilon}{8 + 2\varepsilon}}
\end{align}  for $\alpha,\beta \in [0,1].$ Then, all other strategies $p \notin p_\alpha$ and $q \notin q_\beta$ suffer from a higher worst case exploitability across the two Normal Form Games.
\end{lemma}

\begin{proof}
The idea is that we define a general policy and show that the minimizer of the maximal exploitability across the two Games lies within the considered policy class. We only complete the proof for $q$, it follows analogously for $p$.
\looseness=-1

    We introduce a general policy $\mu = (q,1-q) \in \Delta_2.$ For a general policy $\mu,$ a simple calculation provides the following potential exploitabilities from the perspective of player 1 stated as a function of $q$ across both Games:
   \begin{align}
    &f_1(q) := (2 + 2 \varepsilon)q - 1 - \frac{\varepsilon}{2+\varepsilon},\\
    &f_2(q) := (2+\varepsilon)q - 1 - \frac{\varepsilon/2}{2+\varepsilon/2},\\
    &f_3(q) := 1 - 2q - \frac{\varepsilon}{2+\varepsilon}, \\
    &f_4(q): = 1-2q - \frac{\varepsilon/2}{2+\varepsilon/2}.  %
\end{align}
Comparing $f_3(q)$ and $f_4(q),$ we obtain that they are the same functions except for the $\varepsilon$-term. It holds that $\frac{\varepsilon}{2+\varepsilon} > \frac{\varepsilon/2}{2+\varepsilon/2}$ and therefore $f_3(q) < f_4(q).$ Then, let us define the following convex function,
\begin{align}
    F(q) = \max \{f_1(q), f_2(q), f_4(q)\}.
\end{align}
Note that this function is indeed convex as a maximum of affine functions. 

Next, we show that we can further simplify $F(q).$ Observe that if $f_2(q) > f_1(q)$, then
\[
f_4(q) - f_2(q) = \lp1 - 2q - \tfrac{\varepsilon/2}{2+\varepsilon/2}\rp - \lp(2+\varepsilon)q - 1 - \tfrac{\varepsilon/2}{2+\varepsilon/2}\rp = 2 - (4+\varepsilon)q.
\]
Since $f_2 > f_1$ implies $(2+\varepsilon)q > (2+2\varepsilon)q - \tfrac{\varepsilon}{2+\varepsilon}$, one can check that this forces $q < \tfrac{2}{4+\varepsilon}$, in which case $f_4(q) > f_2(q)$. Hence whenever $f_2$ dominates $f_1$, we automatically have $f_4 > f_2$. Thus the maximum is always realized by either $f_1$ or $f_4$, and we can rewrite
\begin{align}
    F(q) = \max \{f_1(q), f_4(q)\}.
\end{align}

As stated $F(q)$ is the maximum of affine functions, therefore $F$ is convex and has a minimizer (see e.g. \citet[Section 3.2.3]{Boyd_Vandenberghe_2004}). Knowing that $F$ is convex, we can use tools from convex optimization to find $q^* \in \argmin F(q).$ We know that $q^*$ is a minimizer of $F$ if and only if $F$ is subdifferentiable at $q^*$ and 
\[0 \in \partial F (q^*).\]

As a first step, we note that we can write $F(q) = \max \{f_i\}^2_{i=1},$ where $f_i$ is an affine function $\forall i \in \{1, 2\}.$ This implies that the subdifferential at $q^*$ exists and is given by the convex hull of gradients of all active functions at $q^*.$ Next, let us define the index set of all active functions at $q^*$, we get
\begin{align}
    A :=  \{i : f_i(q^*) = F(q^*)\}.
\end{align}
With this definition we have  $0 \in \partial F (q^*) = \mathrm{conv} \{\nabla f_i(q^*): i \in A\}$ and this can equivalently be expressed as $\lambda_i \geq 0$ for $i \in A$ with $\sum \lambda_i = 1$ s.t. 
\begin{align}    
\sum_{i \in A} \lambda_i \nabla f_i(q^*) = 0.
\end{align}

Next, let us compute the gradients with respect to $q$. We get
\begin{align}
    \nabla f_1(q) = 2 + 2 \varepsilon \quad \nabla f_4(q) = -2
\end{align}
Now, let us check that at $q^*$ both functions must be active. Observe that
\[
\nabla f_1(q)=2+2\varepsilon>0,\qquad \nabla f_4(q)=-2<0.
\]
Since neither gradient is zero, no single affine function has zero slope. This implies that  an interior minimizer of the convex function $F(q)$ cannot occur at a point where only one $f_i$ is active. Hence, any interior minimizer must be attained at an intersection where at least two affine pieces are active. In particular, the necessary condition for an interior minimizer here is
\[
f_1(q)=f_4(q).
\]
Let us now check that indeed both functions can be active:
\begin{align}
    \lambda_1 (2 + 2 \varepsilon) - 2\lambda_4 = 0 \Leftrightarrow \lambda_1 (2 + 2 \varepsilon) - 2(1-\lambda_1) = 0 \Leftrightarrow \lambda_1 = \frac{2}{4+2\varepsilon} \in (0,1),
\end{align}
and therefore also $\lambda_4 = 1- \frac{2}{4+2\varepsilon} \in (0,1).$ This indicates that both functions are active at $q^*$, meaning that $A=\{1,4\}.$ 
From this we know that $f_1(q^*) = f_4(q^*)$ and we get 
\begin{align}
    &(2 + 2 \varepsilon)q^* - 1 - \frac{\varepsilon}{2+\varepsilon} = 1-2q^* - \frac{\varepsilon/2}{2+\varepsilon/2}  \\
    \Leftrightarrow & q^* = \frac{2 + \frac{\varepsilon}{2+\varepsilon} - \frac{\varepsilon/2}{2+\varepsilon/2}}{(4 + 2 \varepsilon)} = \frac{2 + \frac{\varepsilon}{2+\varepsilon} - \frac{\varepsilon}{4+\varepsilon}}{2 (2+\varepsilon)}.
\end{align}
It follows that it holds true that $q^* \in (0,1).$ Next, we remind ourselves that in the lower bound proof we considered strategies of the form, with  $\beta \in [0,1]$: 
\[
q=\frac{1}{2} - \lp\frac{2 \beta \varepsilon}{8 + 4 \varepsilon} + \frac{(1-\beta) \varepsilon}{8 + 2\varepsilon}\rp. 
\]

Next, we explicitly calculate $\beta$ to see that indeed the policy lies within $q_\beta$. We remind ourselves, that $q_\beta:=\frac{1}{2} - \brr{\frac{2 \beta \varepsilon}{8 + 4 \varepsilon} + \frac{(1-\beta) \varepsilon}{8 + 2\varepsilon}}.$
Let 
\[
F(\beta) = \frac{1}{2} - \lp\frac{2 \beta \varepsilon}{8 + 4 \varepsilon} + \frac{(1-\beta) \varepsilon}{8 + 2\varepsilon}\rp.
\]
This can be rearranged as
\[
F(\beta) = \frac{1}{2} - \frac{\varepsilon}{8+2\varepsilon} + \beta\lp\frac{\varepsilon}{8+2\varepsilon} - \frac{2\varepsilon}{8+4\varepsilon}\rp.
\]
Now solve $F(\beta) = q^*$. This yields
\[
\beta = \frac{\varepsilon+1}{\varepsilon+2}.
\]
Since for $\varepsilon>0$ we have $\tfrac12 < \tfrac{\varepsilon+1}{\varepsilon+2} < 1$, this $\beta$ indeed lies in $[0,1]$.

Therefore, for $\beta= \tfrac{\varepsilon+1}{\varepsilon+2}$ we recover $q^*$. This, together with the convexity of $F(q)$, shows that the maximal worst case exploitability is always higher for all strategies outside the considered policy class. This completes the proof.  
\end{proof}
Finally, we briefly describe how the proof of Corollary~\ref{cor:lower} is extracted from the proof of Theorem~\ref{thm:lower_bound}.
\begin{proof}[Proof of Corollary~\ref{cor:lower}.]%
In the proof, for Theorem~\ref{thm:lower_bound} we obtained
\begin{align*}
\max_{\mathcal{G}\in\mathcal{H}} \mathbb{E}\bs{\mathrm{Exploitability}^\nu_{\mathcal{G}}(\mathrm{Alg}(\mathcal{D}_{\mathcal{G}}))} &\geq \frac{\varepsilon}{30} \exp\lp-N \mathrm{KL}\lp
\nu_{\mathrm{Nash},\mathcal{G}_1}, \nu_{\mathrm{Nash},\mathcal{G}_2}
    \rp\rp,     
\end{align*}
where $N$ is the number of visits in $s_3$. However, if $\gC_{\max}=\infty$, then $N=0$ because $s_3$ is never visited. This implies
\begin{align*}
\max_{\mathcal{G}\in\mathcal{H}} \mathbb{E}\bs{\mathrm{Exploitability}^\nu_{\mathcal{G}}(\mathrm{Alg}(\mathcal{D}_{\mathcal{G}}))} &\geq \frac{\varepsilon}{30}.
\end{align*}
Repeating the same steps for the other player, and setting $\varepsilon=1/4$ which is the largest possible value that ensures that the payoffs are bounded in $[0,1]$ yields
\begin{align*}
\max_{\mathcal{G}\in\mathcal{H}} \mathbb{E}\bs{\mathrm{Nash\text{-}Gap}_{\mathcal{G}}(\mathrm{Alg}(\mathcal{D}_{\mathcal{G}}))} &\geq \frac{1}{60}.
\end{align*}
At this point, let us consider that after playing one action in $s_3$, the agents move in another state which has exactly the same reward matrices of $s_3$, i.e. $R_{\gG_1}$ and $R_{\gG_2}$. The same transition is repeated for $H-1$ times to ensure that the game has horizon $H$. Let us denote these $H$ steps games $\gG'_1$, $\gG'_2$ and the class $\mathcal{H}':=\bcc{\gG'_1, \gG'_2}$. In this game, we then have 
\begin{align*}
\max_{\mathcal{G}'\in\mathcal{H}'} \mathbb{E}\bs{\mathrm{Nash\text{-}Gap}_{\mathcal{G}'}(\mathrm{Alg}(\mathcal{D}_{\mathcal{G}'}))} &\geq \frac{H-1}{60}.
\end{align*}
\end{proof}

\section{Omitted Proofs for \ours}

In this section, we provide a summary of the main steps used for our main result (Theorem~\ref{thm:main_result}), before we give all the missing details for the analysis of \ours. After the summary, we give the pseudo-code of EULER, that we use in the reward-free warm-up phase. Then, we give the missing proof of Lemma~\ref{lemma:expl_decomposition} and last, we provide the concentration result used for the BC part of the algorithm.

We remind ourselves, that the analysis of the algorithm can be divided into two main steps. The first step concerns the reward-free warm-up phase. Here, we consider the expert-induced MDPs (Definition~\ref{def:induced_MDP}), which are constructed using access to the queriable experts. Informally, within these induced MDPs, we build datasets $(\gD^{\nuE}, \gD^{\muE})$ that provide sufficient coverage of the relevant states. This result is formalized in Theorem~\ref{thm:reward_free_main_result}.

In the second step, these datasets are used to recover the Nash equilibrium policies. Specifically, we apply Behavior Cloning on $\gD^{\nuE}$ to approximate the expert policy $\nuE$ and on $\gD^{\muE}$ to approximate $\muE$. To establish that this procedure effectively minimizes the $\Nashgap$, we rely on Lemma~\ref{lemma:expl_decomposition}, which decomposes the $\Nashgap$ in a way that leverages the dataset distributions. Finally, we invoke Lemma~\ref{lemma:bc_concentration} to bound the concentration error of Behavior Cloning on both datasets, thereby completing the proof.

\subsection{EULER algorithm}
First, for completeness reason we state the EULER pseudo-code. EULER was introduced by \citet{zanette2019tighterproblemdependentregretbounds}. In our context it is used for the reward-free warm-up phase. In particular, it is used to solve the $SH$ RL problems that maximize the probability to reach a certain state. The full pseudo-code is given in Algorithm~\ref{alg:euler}.

\begin{algorithm}[h!]
\caption{EULER($r, N_0, P$)}
\label{alg:euler}
\begin{algorithmic}[1]
\Require Reward function $r$, episodes $N_0$, environment dynamics $P = \bcc{P_h}^H_{h=1}$ .
    \State \textbf{Initialize:} $\delta'=\frac{1}{7} \delta$, 
    ,
    $B_p = H \sqrt{2 \ln \frac{(4SAN_0)}{\delta'}}$, 
    $B_v = \sqrt{2 \ln \frac{(4SAN_0)}{\delta'}}$, 
    $J = H \ln \frac{(4SAN_0)/\delta'}{3}$
\State \textbf{Initialize:} $\pi^1_h = \mathrm{Uniform}(\gA)$ for all $h \in [H]$, for all $s\in\gS$.
\For{$k = 1, 2, \ldots, N_0$}
\State Sample a trajectory $(s^k_1, a^k_1, \dots, s^k_H, a^k_H )$ with policy $\pi^k$ in the environment with dynamics $P$.
\State Set $V^k_{H+1}=0$.
    \For{$h = H, H - 1, \ldots, 1$}
        \State $N^k_h(s',s,a) =\sum^k_{\tau=1} \mathds{1}\bcc{s^\tau_{h+1}, a^\tau_h,s^\tau_h = s',a,s}$
        \State $N^k_h (s,a) = \sum_{s'\in \gS} N^k_h(s',s,a) $.
        \State $\hat{P}^k_h(s'|s,a) = \frac{N^k_h(s', s,a)}{N^k_h(s,a)}$.
        \State $b^k_h(s,a) = \sqrt{\frac{2 \widehat{Var}_{\hat{P}^k_h} (\bar{V}^{k}_{h+1})  (s,a)\ln \frac{4SAT}{\delta'}}{N^k_h(s,a)}} + \frac{H \ln \frac{4SAT}{\delta'}}{3(N^k_h(s,a)-1)}$ where $\widehat{Var}_{\hat{P}^k_h} (V)  (s,a) = \hat{P}^k_h(V - \hat{P}^k_h V(s,a))^2(s,a)$
                \State $B^h_k(s,a) = b^h_k(s,a) + \frac{1}{\sqrt{N^k_h(s,a)}} \left( \frac{4J+B_p}{\sqrt{N^k_h(s,a)}} + B_v \| \bar{V}^k_{h+1} - \underline{V}^k_{h+1} \|_{2, \hat{P}^k_h} \right)$
                \State $Q^k_h(s,a) = \min\left\{H - h, r_h(s,a) + \hat{P}^k_h \bar{V}^k_{h+1} (s,a) + B^k_h(s,a) \right\}$
            \State $\pi^k_h(s) = \arg \max_{a \in \gA} Q^k_h(s,a)$
            \State $\bar{V}^k_h(s) = \max_{a\in \gA}Q^k_h(s,a)$
            \State $\underline{V}^k_h(s) = \max\left\{0, r_h(s,a) + \hat{P}^k_h \underline{V}^k_{h+1}(s,a) - B^k_h(s,a)\right\}$
    \EndFor
\EndFor
\end{algorithmic}
\end{algorithm}

\subsection{Proof of Lemma~\ref{lemma:expl_decomposition}}
For proving this result, the main idea is to decompose the exploitability into the Total variation between the estimated expert and the true one. Next, we will restate the Lemma.

\begin{lemma}[\textbf{Exploitability decomposition}]
For any policy pair $\nu,\nu'$, we define their total variation at state $s$ as $\TV(\nu,\nu')(s) = \sum_{b \in \mathcal{B}} \abs{\nu(b|s)-\nu'(b|s)}.$ It holds that
\begin{align*}
\innerprod{\initial}{V^{\mu^{\star}, \widehat{\nu} }-   V^{ \widehat{\mu}, \nu^{\star}} } \leq  2H \sum^H_{h=1} \sum_{\pi \in \bcc{\hat{\mu},\hat{\nu}}}\mathrm{Err}_h(\pi).
\end{align*}

where we have defined $\mathrm{Err}_h(\hat{\nu}):=\max_{\mu_h \in \mathrm{br}(\widehat{\nu}_h)}\expect_{ s\sim d^{\mu, \nuE}_h} \ls \TV \lp \nuE_h, \widehat{\nu}_h \rp (s) \rs $ for player 1 and additionally $\mathrm{Err}_h(\hat{\mu}):=\max_{\nu_h \in \mathrm{br}(\widehat{\mu}_h)}\expect_{  s\sim d^{\muE,\nu}_h} \ls \TV \lp \muE_h , \widehat{\nu}_h \rp (s) \rs$ for player 2.
\end{lemma}
\begin{proof}
We start by upper bounding the decomposition as follows 
\begin{align*}
\innerprod{\initial}{V^{\mu^{\star}, \widehat{\nu} } -   V^{ \widehat{\mu}, \nu^{\star}}} &=  \innerprod{\initial}{V^{\mu^{\star}, \widehat{\nu} } - V^{\muE,\nuE}} + \innerprod{\initial}{V^{\muE,\nuE} - V^{ \widehat{\mu}, \nu^{\star}}}\\
   &\leq \innerprod{\initial}{V^{\mu^{\star}, \widehat{\nu} } - V^{\mu^{\star},\nuE}} + \innerprod{\initial}{V^{\muE,\nu^{\star}} - V^{ \widehat{\mu}, \nu^{\star}}},
\end{align*}
where $\mu^\star$ and $\nu^\star$ are arbitrary policies in the sets $\mathrm{br}(\hat{\nu})$ and $\mathrm{br}(\hat{\mu})$ respectively.
At this point, we identified two pairs of value function differences where one policy is fixed, respectively $\mu^\star$ and $\nu^\star$.
Therefore, applying the performance difference lemma (see e.g. \citet{Kakade2002ApproximatelyOA}) in the MDP induced by $\mu^\star$ we obtain
\[
\innerprod{\initial}{V^{\mu^{\star}, \widehat{\nu} } - V^{\mu^{\star},\nuE}} \leq \sum^H_{h=1} \mathbb{E}_{s \sim d^{\mu^\star,\nuE}_h} \bs{\innerprod{Q^{\mu^\star,\hat{\nu}}(s,\cdot)}{\hat{\nu}(\cdot|s) - \nuE(\cdot|s)}}.
\]
Then, by Hölder's inequality with $\lnorm \cdot\rnorm_1 $ and $\lnorm \cdot\rnorm_\infty $ and additionally bounding the value function with its maximal value $H,$ it holds that
\[
\innerprod{\initial}{V^{\mu^{\star}, \widehat{\nu} } - V^{\mu^{\star},\nuE}} \leq H \sum^H_{h=1} \mathbb{E}_{s \sim d^{\mu^\star,\nuE}_h} \bs{\TV\brr{\hat{\nu}(\cdot|s), \nuE(\cdot|s)}}.
\]
Then, since we aim for a bound on the left hand side that holds for any $\mu^\star \in \mathrm{br}(\hat{\nu})$ we need to pick the maximizer over the right hand side.
\[
\innerprod{\initial}{V^{\mu^{\star}, \widehat{\nu} } - V^{\mu^{\star},\nuE}} \leq H \max_{\mu \in \mathrm{br}(\hat{\nu})}\sum^H_{h=1} \mathbb{E}_{s \sim d^{\mu,\nuE}_h} \bs{\TV\brr{\hat{\nu}(\cdot|s), \nuE(\cdot|s)}}
\]
Equivalent steps for the second player to upper bound $\innerprod{\initial}{V^{\muE,\nu^{\star}} - V^{ \widehat{\mu}, \nu^{\star}}}$ concludes the proof.
\end{proof}

\subsection{Behavior Cloning concentration}

Last, we adapt the analysis of Lemma D.1 by \citet{freihaut2025learningequilibriadataprovably} to our setting. The main changes are that the expectation used in our setting is with respect to the dataset distributions $p^{\nuE}$ and $p^{\muE}$ while their is with respect to the expert occupancy measure $d^{\muE,\nuE}$.

\begin{lemma}
\label{lemma:bc_concentration}
Let $p^{\nuE}_h$ and $p^{\muE}_h$ be two distributions received from running the reward-free warm-up of Algorithm \ours{} and $N$ the size of the received datasets. Then, for all $s\in \gS^{\nuE}_{\delta,h}$ it holds with probability of at least $1-\delta/2$ that
\begin{align}
    \mathbb{E}_{s \sim p^{\nuE}_h}\bs{\TV \lp   \nuE_h, \widehat{\nu}_h \rp (s)} \leq \sqrt{\frac{SB\log(4S/\delta)}{N}}.
\end{align}
Similarly, for all $s\in \gS^{\muE}_{\delta,h}$ it holds with probability of at least $1-\delta/2$ that
\begin{align}
    \mathbb{E}_{s \sim p^{\muE}_h}\bs{\TV \lp   \muE_h, \widehat{\mu}_h \rp (s)} \leq \sqrt{\frac{SA\log(4S/\delta)}{N}}.
\end{align}
\end{lemma}
\begin{proof}
    We only provide the proof for the distribution $p_h^{\nuE},$ it follows analogously for $p_h^{\muE}.$ 
    
We get
    \begin{align*}
    \expect_{s\sim p_h^{\nuE}} \ls  \TV \lp   \nuE_h,  \widehat{\nu}_h  \rp(s)  \rs &\overset{\mathrm{(i)}}{\leq}   \sum_{s \in \gS} p^{\nuE}_h (s) \sqrt{\frac{2 B \log (4S /\delta)}{\max \{ N (s), 1 \}}}
    \\
    &=  \sum_{s \in \gS} \sqrt{p_h^{\nuE} (s)}  \sqrt{\frac{2 B\, p_h^{\nuE} (s) \log (4S /\delta)}{\max \{ N(s), 1 \}}}
    \\
    &\overset{\mathrm{(ii)}}{\leq} \sqrt{\sum_{s \in \gS} \frac{2 B \, p_h^{\nuE} (s) \log (4S /\delta)}{\max \{ N(s), 1 \}} }
    \\
    &\overset{\mathrm{(iii)}}{\leq} \sqrt{\sum_{s \in \gS} \frac{16 B  \log^2 (4S /\delta) }{N} }
    \\
    &= 4 \sqrt{\frac{SB  \log^2 (4S /\delta) }{N}},
\end{align*}
where in $\mathrm{(i)}$ we applied Lemma~\ref{lemma:l1_concentration} and a union bound over the state space $\gS$, in $\mathrm{(ii)}$ we applied Cauchy Schwarz and in $\mathrm{(iii)}$ we applied Lemma~\ref{lemma:binomial-concentration}, reminding ourselves that $N$ is the size of the dataset.
\end{proof}

\subsection{Extension to n-player general-sum Markov games}
\label{appendix:general_sum}
 
In this paragraph, we show that our approach is easily extendable to $n$-player general-sum games.  \citet{freihaut2025learningequilibriadataprovably}  also provide an extension for this setting, however their proof is hard to parse while our approach translates to the $n$-player setting smoothly. First, we introduce all the necessary notation.
\paragraph{Notation General-Sum Markov games}
A general-sum Markov game is defined by the tuple $\gG = (n, H,\gS, \gA, P, r, d_0),$ where compared to zero-sum games, $n$ is now the number of players, $\gA:= \gA_1 \times ..\times \gA_n$ the joint action space composed of the individual action spaces $\gA_i;$ the reward function for the $i^{th}$ player at stage $h \in [H]$ is $r_{h,i}: \gS \times \gA \rightarrow [0,1]$ and $P$ is now the transition function that takes a joint action $a \in \gA$ as an input. We denote a joint policy as $\pi:= (\pi_1, \ldots, \pi_n)$, where $\pi_i: \gS \to \Delta_{\gA_i},$ where $\Delta_{\gA_i}$ is the probability simplex over the individual action space $\gA_i$. We denote the set of Markov policies for the agent $i$ as $\Pi_i$. Further let us denote the cardinality of an individual action space as $A_i := \abs{\gA_i}$ and the maximal cardinality of the individual state spaces as $A_{\max} := \max_i \abs{\gA_i}$ for $i \in [n].$ Additionally, we will make use of the convention that $\pi_{-i}$ denotes the policy of all agents but policy of agent $i$. The same notation is also used for an action $a_{-i} = (a_1,\ldots,a_{i-1}, a_{i-1}, \ldots, a_n)$ and the according action space is denoted by $\gA_{-i}.$
At this point, we can define the value function of the policy profile $\pi$ at stage $h \in [H]$ for a certain agent $i \in [n]$ as $V_{h,i}^{\pi}(s) := \mathbb{E}_{\pi}\left[ \sum_{t = h}^{H-1} r_{i,t}(S_t, A_1, \dots, A_n) \,\middle|\, S_h = s \right]$. Often times, we will shorten $V_{0,i}^{\pi}$ by $V_{i}^{\pi}$. Having defined the value function, we can introduce the Nash gap for general sum games as follows
\begin{equation}
\mathrm{Nash\text{-}Gap}(\pi) = \max_{i \in [n]} \max_{\pi'_i: \gS \rightarrow \Delta_{\A_i}} \langle \initial,  V_{i}^{\pi'_i,\pi_{-i}} - V_{i}^{\pi} \rangle. \label{eq:multi_player_nash_gap}
\end{equation}
Note that in general-sum games, the value of the NE does not need to be unique and the set of NE is not convex in general.

\paragraph{Results for the $n$-player setting}
With the given notation, let us redefine the expert induced MDP for this setting. 

\begin{definition}[Experts Induced MDP]
\label{def:induced_MDP_general}
    Let $\gG$ be a Markov game and $\pi^E$ the expert policies, then $\mathcal{M}^{\pi^E_{-i}}:= (\gS, \gA, P^{\pi^E_{-i}}, r^{\pi^E_{-i}}, H)$ is the MDP induced by the experts $\pi^E_{-i}$ with the transition model $P_h^{\pi^E_{-i}} (s' \mid s,a):= \sum_{a_{-i}\in\gA_{-i}}  \pi^E_{{-i}, h}(a_{-i} \mid s) P_h(s'\mid s,a,a_{-i}) $ and an arbitrary reward function $r^{\pi^E_{-i}}_{i,h}(s,a_i) \in \{0,1\} \quad \forall (s,a_i) \in \gS \times \gA_i.$
\end{definition}

Now, we can restate Algorithm~\ref{alg:mail_warm_general} for the n-player general-sum setting.

\begin{algorithm}[!h]
\caption{Multi-Agent Imitation Learning with reward-free warm-up (\ours) for n-player Games}
\label{alg:mail_warm_general}
\begin{algorithmic}[1]
\State \textbf{Input:} iteration number $N_0$, $N$, queriable experts $\pi^E$.
\State \textbf{Reward-free warm-up phase:}
\For{all $i \in [n]$}
\State set policy class $\Psi^{\pi^E_{-i}} \leftarrow \emptyset$, and dataset $\mathcal{D} \leftarrow \emptyset$.
\For{all $(s,h) \in \mathcal{S} \times [H]$}
    \State $r_{i,h}^{\pi^E_{-i}}(s', a_i') \leftarrow \mathbf{1}[s'=s \text{ and } h'=h]$ for all $(s',a_i',h') \in \mathcal{S} \times \mathcal{A} \times [H]$.
    \State $\bcc{\pi_i^{(s,h)}}^{N_0}_{i=1} \leftarrow \text{EULER}(r^{\pi^E_{-i}}, N_0, P^{\pi^E_{-i}})$.
    \State Let $\Phi^{(s,h)} \gets \bcc{\pi_i^{(s,h)}}^{N_0}_{i=1}$
    \State $\pi_{i,h'}(\cdot|s) \leftarrow\text{Unif}(\mathcal{A}_i)$, $~\forall \mu \in \Phi^{(s,h)}, \forall h' \geq h $.
    \State $\Psi^{\pi^E_{-i}} \leftarrow \Psi^{\pi^E_{-i}} \cup \Phi^{(s,h)}$.
\EndFor
\For{$n = 1 \dots N$}
    \State sample policy $\pi_i \sim \text{Unif}(\Psi^{\pi^E_{-i}})$.
    \State Collect $ z_n = (s_1, a_1, a_{-i}, \dots, s_{H+1})  \sim \pi_i, \pi^E_{-i}$.
    \State $\mathcal{D}^{\pi^E_{-i}} \leftarrow \mathcal{D}^{\pi^E_{-i}} \cup \{z_n\}$
\EndFor
\EndFor
\State \textbf{Receive:} datasets $\mathcal{D}^{\pi^E_{-i}}$ for all $i \in [n]$.
\State \textbf{Imitation Learning}
\For{$i \in [n]$}
\State  Define the dataset $\mathcal{D}^{\pi^E_i} = \cup_{j \neq i} \mathcal{D}^{\pi^E_{-j}}$ merging  to compute 
\begin{equation*}
\hat{\pi}_{i} = \argmin_{\pi_i\in \Pi_i} \sum_{s,a_{i} \in \mathcal{D}^{\pi^E_{i}} } - \log \pi_{i}(a_{i}|s) %
\end{equation*}
where $a_{i}$'s are sampled from $\pi^E_i(\cdot|s)$.
\EndFor
\State \textbf{Return} Nash estimate $\widehat{\pi} = (\widehat{\pi}_1, \ldots,\widehat{\pi}_n).$
\end{algorithmic}
\end{algorithm}

The reward free warm-up phase again produces a dataset that covers all $\delta$- significant states well.
\begin{corollary}
Let $\mathcal{M}^{\pi^E_{-i}}$ be the induced MDP defined in \ref{def:induced_MDP_general} and the policy set $\Psi^{\pi_{-i}}$ is generated according to Algorithm~\ref{alg:mail_warm_general}. Then there exists  an absolute constant $c>0$ such that for any $\varepsilon > 0 $ and $p\in(0,1),$ if we set $N_0 \geq n^2cS^2AH^4\iota^3_0/\delta,$ where $\iota_0:= \log(SAH/p\delta),$ then with probability $1-p,$ the reward free exploration returns a sampling distribution $p^{\pi_{-i}^E}$ such that for all $\pi_i \in \Pi$
\begin{align}
\label{eq:reward_free_coverage_general}
        \forall \delta-\mathrm{significant} (s,h), \quad \max_{a,h} \frac{d^{\pi_i,\pi^E_{-i}}_h(s,a)}{p_h^{\pi^E_{-i}}(s,a)} \leq 2SA_iH,
\end{align}
    where $\delta-\mathrm{significant},$ means that the probability to reach a state $s$ in the induced MDP $\mathcal{M}^{\pi^E_{-i}}$ is lower bounded by $\delta$:
    \looseness=-1
    \[\max_{\pi_i} d^{\pi_i,\pi^E_{-i}}_h \geq \delta.\]
\end{corollary}

Next, we can state the guarantees of \ours{} for the $n$-player general-sum setting.

\begin{corollary}
\label{thm:main_result_general}
    For any $\varepsilon > 0$ and $\delta_{\mathrm{fail}} \in (0,1)$ if we execute Algorithm~\ref{alg:mail_warm_general} and choose the parameters according to $N = \gO(\frac{n^{4}H^6S^3A^3_{\max} \log(S/\delta_{\mathrm{fail}})}{\varepsilon^2})$ and $N_0 \geq \mathcal{O}\brr{n^3S^3A^2_{\max}H^6\iota^3_0/\varepsilon}$, we get with probability $1-\delta_{\mathrm{fail}}$ for the policies $\widehat{\pi}$ that
    \[\Nashgap(\widehat{\pi}) \leq \mathcal{O}(\varepsilon).\]
\end{corollary}
\begin{proof}

Let us first remind ourselves that the Nash Gap for the $n$-player general-sum setting is slightly different. In particular, we have 
\begin{align*}
\mathrm{Nash\text{-}Gap}(\hat{\pi})
&= \sum_{i=1}^n \max_{\pi'_i}\; \innerprod{\initial}{V^{\pi'_i, \hat{\pi}_{-i}}_i - V^{\hat{\pi}_i, \hat{\pi}_{-i}}_i } \\
&= \sum_{i=1}^n \innerprod{\initial}{V^{\pi^\star_i, \hat{\pi}_{-i}}_i - V^{\pi^E_{i}, \pi^E_{{-i}}}_i}
\;+\;\innerprod{\initial}{V^{\pi^E_{i}, \pi^E_{{-i}}}_i - V^{\hat{\pi}_i, \hat{\pi}_{-i}}_i } \\
&\le \underbrace{\sum_{i=1}^n 
\innerprod{\initial}{\,V^{\pi^\star_i, \hat{\pi}_{-i}}_i - V^{\pi^\star_i, \pi^E_{{-i}}}_i\,}}_{:= \mathrm{Exploit\text{-}Gap}}
\;+\;
\underbrace{\sum^n_{i=1}\innerprod{\initial}{\,V^{\pi^E_{i}, \pi^E_{{-i}}}_i - V^{\hat{\pi}_i, \hat{\pi}_{-i}}_i\,}}_{:=\mathrm{Value\text{-}Gap}}.
\end{align*}
In the first step of the proof, we consider $\mathrm{Exploit\text{-}Gap}$, which is the part where we require the reward-free warm-up phase. 

In particular, we get
\begin{align*}
&\mathrm{\mathrm{Exploit\text{-}Gap}}
\overset{(i)}{\leq} H\; \sum^H_{h=1}\sum_{i=1}^n
\max_{\pi_i \in \mathrm{br}(\widehat{\pi}_{-i})}
\mathbb{E}_{s \sim d^{\pi_i, \pi^E_{-i}}_h}\ls
\TV\!\big(\pi^E_{-i,h}(\cdot\mid s), \widehat{\pi}_{-i,h}(\cdot\mid s)\big)\rs
 \\
&=
H\;\sum^H_{h=1} \sum_{i=1}^n \max_{\pi_i \in \mathrm{br}(\widehat{\pi}_{-i})}
\left(\sum_{s \in \mathcal{S}^i_{\delta,h}} \sum_{a\in \mathcal{A}} \frac{d^{\pi_i,\pi^E_{-i}}_h(s,a)}{p_h^{\pi^E_{-i}}(s,a)} \TV \lp   \pi^E_{-i,h}, \widehat{\pi}_{-i,h} \rp (s) \right.\\
&\left.\qquad\qquad\qquad\qquad\qquad\quad+\sum_{s \notin \mathcal{S}^i_{\delta,h}} \sum_{a\in \mathcal{A}} d^{\pi_i,\pi^E_{-i}}_h(s,a) \TV \lp   \pi^E_{-i,h}, \widehat{\pi}_{-i,h} \rp (s) \right) \\
&\overset{(ii)}{\leq} 2H^2SA_{\max}\;
\sum^H_{h=1}\sum_{i=1}^n \mathbb{E}_{s \sim p_h^{\pi^E_{-i}}}\Bigg[\sum_{j \neq i}
\TV\!\big(\pi^E_{j,h}(\cdot \mid s), \widehat{\pi}_{j,h}(\cdot \mid s)\big)\Bigg]
+ nH^2SA_{\max}\,\delta \\
&= 2H^2SA_{\max}\;
\sum^H_{h=1}\sum_{i=1}^n \sum^n_{j \neq i} \mathbb{E}_{s \sim p_h^{\pi^E_{-j}}}\Bigg[
\TV\!\big(\pi^E_{i,h}(\cdot \mid s), \widehat{\pi}_{i,h}(\cdot \mid s)\big)\Bigg]
+ nH^2SA_{\max}\,\delta \\
&\leq 2n^2SA_{\max}H^3 \; \sqrt{\frac{SA_{\max} \log(4S/\delta_{\mathrm{fail}})}{nN}}
+ nSA_{\max}H^2\,\delta \\
&=  2n^{\frac{3}{2}}SA_{\max}H^3 \; \sqrt{\frac{SA_{\max} \log(4S/\delta_{\mathrm{fail}})}{N}}
+ nSA_{\max}H^2\,\delta,
\end{align*}
where in $(i)$ we used the same argument as in Lemma~\ref{lemma:expl_decomposition} for the $n$ player setting. In $(ii)$ we used the fact that the policies are conditionally independent on $s$. This means that we now bound the TV separately for each player and get in total $(n-1)$ independent bounds, where we again can apply for example \cite[Theorem 2.1]{concentrationofmissingmass}. Moreover, we defined as $\mathcal{S}^i_{\delta,h}$ the set of $\delta$-reachable states at stage $h$ in the MDP $\mathcal{M}^{\pi^{\expert}_{-i}}$.

For the value gap we can proceed similarly,
\begin{align*}
\mathrm{Value\text{-}Gap} &\overset{(i)}{\leq} n H \sum^H_{h=1}\;
\mathbb{E}_{s \sim d^{\pi^E}_h}\ls
\TV\!\big(\pi_h^E(\cdot\mid s), \widehat{\pi}_h(\cdot\mid s)\big)\rs
 \\
&= n H \sum^H_{h=1}\; \sum_{i=1}^n
\mathbb{E}_{s\sim d^{\pi^E}_h}\ls
\TV\!\big(\pi_{i,h}^E(\cdot\mid s), \widehat{\pi}_{i,h}(\cdot\mid s)\big)\rs\\
&=
n H \; \sum^H_{h=1} \sum_{i=1}^n
\left(
\sum_{s \in \mathcal{S}^i_{\delta,h}} \sum_{a\in \mathcal{A}_i} \frac{d^{\pi^E_{i},\pi^E_{-i}}_h(s,a)}{p_h^{\pi^E_{-i}}(s,a)} \TV \lp   \pi^E_{i,h}, \widehat{\pi}_{i,h} \rp (s) \right.\\
& \left.\qquad\qquad\qquad\qquad+\sum_{s \notin \mathcal{S}^i_{\delta,h}} \sum_{a\in \mathcal{A}_i} d^{\pi^E_{i,h},\pi^E_{-i,h}}_h(s,a) \TV \lp   \pi^E_{i,h}, \widehat{\pi}_{i,h} \rp (s) \right) \\
&\overset{(ii)}{\leq} 2nH^2SA_{\max}\;
\sum^H_{h=1}\sum_{i=1}^n \mathbb{E}_{s \sim p^{\pi^E_{-i}}}\Bigg[
\TV\!\big(\pi^E_{i,h}(\cdot \mid s), \widehat{\pi}_{i,h}(\cdot \mid s)\big)\Bigg]
+ nH^2SA_{\max}\,\delta \\
&= 2nH^2SA_{\max}\;
\sum^H_{h=1}\sum_{i=1}^n \mathbb{E}_{s \sim p^{\pi^E_{-i}}}\Bigg[
\TV\!\big(\pi^E_{i,h}(\cdot \mid s), \widehat{\pi}_{i,h}(\cdot \mid s)\big)\Bigg]
+ nH^2SA_{\max}\,\delta \\
&\leq 2n^2SA_{\max}H^3 \; \sqrt{\frac{SA_{\max} \log(4S/\delta_{\mathrm{fail}})}{N}}
+ nSA_{\max}H^2\,\delta \\
&= 2n^2SA_{\max}H^3 \; \sqrt{\frac{SA_{\max} \log(4S/\delta_{\mathrm{fail}})}{N}}
+ n SA_{\max}H^2\,\delta
\end{align*}
where steps $(i)$ and $(ii)$ holds exactly for the same reasons used in the upper bound of $\mathrm{Exploit\text{-}Gap}$.

Combining both parts completes the proof, giving that the total number of expert queries is given by $\gO(\frac{n^{4}S^3A_{\max}^3 H^6\log(S/\delta_{\mathrm{fail}})}{(1-\gamma)^6\varepsilon^2}).$
\end{proof}

Some remarks are in order. Again this result needs a reward free warm-up phase. In particular, it requires a dataset for each expert, meaning $n$ datasets where each dataset depends on the other $n-1$ agents. Most importantly, we can see that the number of samples needed does not scale exponentially with the number of agents, instead it only scales quadratically. This is in contrast with learning Nash equilibria in the first places, where it is known that the number of samples scales with $A_{\max}^n$, known as the \emph{curse of multi-agents} \citep{rubinstein2016settlingcomplexitycomputingapproximate}. However, this does not contradict with the lower bound as we already have access to data stemming from Nash equilibrium policies which provides additional information. That the lower bound does not hold in these settings has e.g. also been shown in offline general-sum settings \citep{cui2022provablyefficientofflinemultiagent}.

\section{Experimental details}
In this section, we provide details on the experimental setup used for our provided numerical verifications illustrated in Figure \ref{fig:experiments}.

\paragraph{Lower bound environment}
The first experimental environment corresponds to the lower bound construction described earlier (see Figure~\ref{fig:lower}), with a simplification of the game in state $s_3$. Instead of constructing an $\varepsilon$-perturbed Matching Pennies game, we use a normal-form game with a pure Nash equilibrium and unique value of $1$.  
\looseness=-1

Formally, the state space is $\gS=\{s_1,s_2,s_3\}$. The action space is $\gA=\{a_1,a_2\}$ for player 1 and $\gB=\{b_1,b_2\}$ for player 2. The reward function is state-dependent in $s_1$ and $s_2$, with $r(s_1)=r(s_2)=0$. At $s_3$ the reward structure is given by
\[
r(s_3, a, b) := 
\begin{bmatrix}
1 & 1 \\
0 & -12
\end{bmatrix},
\]
where the row indicates the action of player 1 and the column the action of player 2.

The Nash equilibrium strategy for player 1 in this normal-form game is $\mu_{\mathrm{Nash}}(\cdot \mid s_3)=(1,0)$, while player 2 can play any strategy, since her expected reward is always $-1$. The unique Nash value from player 1’s perspective is therefore $1$.  

As in the lower bound construction, in state $s_1$ player 2 strictly prefers $b_1$, which deterministically transitions to $s_2$. Consequently, under the Nash equilibrium profile, only states $\{s_1,s_2\}$ are visited, and the Nash value of the Markov game is $0$.  

If $s_3$ is never visited in the dataset, player 1’s recovered policy by BC will be uniform, which can be exploited by player 2 through the best response
\[
\nu_{\mathrm{br}}(\cdot \mid s_3) = (0,1),
\]
leading to a reward of $-5.5$ for player 1 and $+5.5$ for player 2. In this case, player 1 is exploitable. Conversely, if $s_3$ is covered in the data, then player 1 requires only a single sample to recover the correct Nash strategy, and the exploitability becomes $0$.  

The probability of visiting $s_3$ in any given trajectory is $\rho(s_3)$. Thus, the number of trajectories required until $s_3$ is observed follows a geometric distribution with parameter $\rho(s_3)$, yielding an expected sample complexity of $1/\rho(s_3)$. Since $\gC_{\max} = 1/\rho(s_3)$, varying $\rho(s_3)\in\{1,0.5,0.25,0\}$ corresponds to $\gC_{\max}\in\{1,2,4,\infty\}$, which is exactly reflected in the experimental results shown in Figure~\ref{fig:experiments} (a). On the contrary, notice that for all values of $\rho(s_3)$, the value of $\mathcal{C}(\muE,\nuE)$ remains constant equal to $2$ which is its smallest possible value. By simulating the geometric distribution over $100$ runs across varying seeds and with the different parameters described above and tracking its standard deviation, we exactly recover the plots in Figure~\ref{fig:experiments} (a).

\paragraph{Gridworld}
We next describe the setup of the considered zero-sum Gridworld environment, illustrated in Figure~\ref{fig:grid}. The state space is given by the joint positions of the two agents on a $3\times 3$ grid, subject to the restriction that both agents cannot occupy the same cell simultaneously. Formally,
\[
\gS = \{((i,j),(k,l)) \mid (i,j) \neq (k,l), \, i,j,k,l \in \{0,1,2\}\},
\]
which yields $72$ states in total. The action space is identical for both agents and defined as $\gA=\{\text{left}, \text{right}, \text{up}, \text{down}\}$. The transition dynamics are deterministic: whenever an action would cause an agent to collide with a wall or with the other agent, the agent remains in its current position.  

If the initial distribution is chosen such that both agents are equidistant to the goal, the Nash value of the game is $0$. In particular, we fix the deterministic starting state $((1,0),(1,2))$, from which both players require exactly three steps to reach the goal. Hence, the Nash equilibrium value is $0$. Multiple Nash equilibria exist: any pair of paths in which both players ensure that the opponent cannot reach the goal earlier constitutes a Nash equilibrium. This is reflected in Figure~\ref{fig:grid}, which illustrates different Nash paths obtained using zero-sum value iteration. Importantly, these policies also ensure that in all other states, the opponent cannot force an earlier goal arrival. For the second Gridworld experiment, we take convex combinations of different Nash paths to improve state coverage. Since the set of Nash equilibria in zero-sum games is convex, all such convex combinations remain valid Nash equilibria.  

Both Gridworld experiments use the same environment specification. The only difference lies in the coverage provided by the expert demonstrations. In both cases, the expert policy is obtained by running zero-sum value iteration, which returns a Nash equilibrium policy pair.  

\begin{figure}

    \centering
    \begin{minipage}{0.23\textwidth}
    \begin{tikzpicture}
        \def\gridsize{3}

        \foreach \x in {0,...,\gridsize} {
            \foreach \y in {0,...,\gridsize} {
                \draw[thin, gray] (\x, 0) -- (\x, \gridsize);
                \draw[thin, gray] (0, \y) -- (\gridsize, \y);
            }
        }
        \fill[blue!30] (2, 2) rectangle (3, 3);

        \filldraw[red] (0.5, 1.5) circle (0.2);

        \filldraw[green] (1.5, 0.5) circle (0.2);

    \end{tikzpicture}
 \end{minipage}
 \begin{minipage}{0.23\textwidth}
    \begin{tikzpicture}
        \def\gridsize{3}

        \foreach \x in {0,...,\gridsize} {
            \foreach \y in {0,...,\gridsize} {
                \draw[thin, gray] (\x, 0) -- (\x, \gridsize);
                \draw[thin, gray] (0, \y) -- (\gridsize, \y);
            }
        }
        \fill[blue!30] (2, 2) rectangle (3, 3);

        \filldraw[red] (0.5, 1.5) circle (0.2);

        \filldraw[green] (1.5, 0.5) circle (0.2);

        \draw[->, thick, green] (1.5, 0.75) -- (1.5, 1.3);
        \draw[->, thick, green] (1.5, 1.5) -- (1.5, 2.3);

        \draw[->, thick, red] (0.5, 1.75) -- (0.5, 2.3);
        \draw[->, thick, red] (0.75, 2.5) -- (1.4, 2.5);

    \end{tikzpicture}
    \end{minipage}
 \begin{minipage}{0.23\textwidth}
    \begin{tikzpicture}
        \def\gridsize{3}

        \foreach \x in {0,...,\gridsize} {
            \foreach \y in {0,...,\gridsize} {
                \draw[thin, gray] (\x, 0) -- (\x, \gridsize);
                \draw[thin, gray] (0, \y) -- (\gridsize, \y);
            }
        }
        \fill[blue!30] (2, 2) rectangle (3, 3);

        \filldraw[red] (0.5, 1.5) circle (0.2);

        \filldraw[green] (1.5, 0.5) circle (0.2);

        \draw[->, thick, green] (1.75, 0.5) -- (2.25, 0.5);
        \draw[->, thick, green] (2.5, 0.5) -- (2.5, 1.25);
        \draw[->, thick, green] (2.5, 1.5) -- (2.5, 2.25);

        \draw[->, thick, red] (0.75, 1.5) -- (1.25, 1.5);
        \draw[->, thick, red] (1.5, 1.75) -- (1.5, 2.25);
        \draw[->, thick, red] (1.65, 2.5) -- (2.4, 2.5);
\draw[->, thick, green] (2.5, 0.5) -- (2.5, 1.25);
\end{tikzpicture}
\end{minipage}
\caption{Zero-sum Gridworld environment and different Nash equilibrium paths.}
\label{fig:grid}
\end{figure}
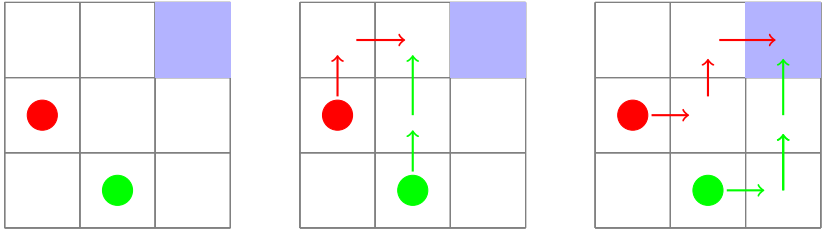

\paragraph{Algorithm setup}
We next detail the implementation of MURMAIL and MAIL-WARM. For MURMAIL, we closely follow the description of \citet{freihaut2025learningequilibriadataprovably}, and restate their pseudocode for completeness (see Algorithm \ref{alg:name}). Note that all experiments have been run on a standard MacBook Pro with Chip M3 and 16GB of RAM. Since the Gridworld environment is more challenging than the lower bound construction considered in their work, we apply two modifications to improve convergence speed: (i) we set the learning rate to $\eta=50$, and (ii) instead of sampling a single policy per RL inner loop, we average over $100$ updates, effectively yielding a batched variant of MURMAIL better suited for larger environments. We set the inner RL loop horizon to $T=10$.  

For \ours, we follow Algorithm~\ref{alg:mail_warm}, with the only practical adjustment being that the expert policies are non-stationary. Consequently, in Line 12 we return stationary approximations of the learned policies. Since EULER is not well suited for practice, and as \citet{jin2020reward} show that any RL algorithm can be used to solve the $SH$ many RL problems in the reward-free phase, we instead employ Q-learning. Importantly, Q-learning does not require knowledge of the transition dynamics. We run Q-learning for $100$ iterations for each RL problem in our experiments.

An interesting empirical observation is that solving the $SH$ many RL problems reveals that many states in the induced expert MDP are not reachable. For example, consider the Nash equilibrium policy illustrated in the middle of Figure \ref{fig:grid}. Fixing the green agent for the expert induced MDP, all states such that the position of the green agent is the bottom right corner are not reachable. Therefore, the effective visited states can significantly reduce compared to the whole state space.

\begin{algorithm}[!h]
\caption{Maximum Uncertainty Response Multi-Agent Imitation Learning (MURMAIL)}
\label{alg:name}
\begin{algorithmic}[1]
\State \textbf{Input:}number of iterations $K$, learning rates $\eta$, inner iteration budget $T$, initial $(\mu_1,\nu_1)$
\State \textbf{Receive:}$\varepsilon$-Nash equilibrium $(\hat{\mu},\hat{\nu})$
\For{$k = 1$ \textbf{to} $K$}
    \textbf{Inner Single-Agent RL Updates:}\\
    \textcolor{blue}{\% Maximum uncertainty response to $\mu$-player update} \\
        Define single agent transition  $P_{\mu_k}(s'\mid s,b) = \sum_{a \in \mathcal{A}} \mu_k(a\mid s) P(s'\mid s,a,b)$; \\
        Define single agent stochastic reward
        $R_{\mu_k}(s) \rightarrow \mathds{1}_{\{A_E = A_E'\}} - 2 \mu_k(A_E\mid s) + \| \mu_k(\cdot|s)\|^2$ where $A_E, A_E' \sim \muE(\cdot \mid s)$;\\
     $y_k = \texttt{UCBVI}(T, P_{\mu_k}, R_{\mu_k})$;\\
    \textcolor{blue}{\% Maximum uncertainty response to $\nu$-player update} \\
         $P_{\nu_k}(s'|s,a) = \sum_{b \in \mathcal{B}} \nu_k(b|s) P(s' \mid s,a,b)$; \\
        $R_{\nu_k}(s) \rightarrow \mathds{1}_{\{A_E = A_E'\}} - 2 \nu_k(A_E \mid s) + \| \nu_k(\cdot \mid s)\|^2$ where $A_E, A_E' \sim \nuE(\cdot \mid s)$;\\
     $z_k = \texttt{UCBVI}(T, P_{\nu_k}, R_{\nu_k})$ \\
    \textbf{Update policies:}\\
    Sample $S^\mu_k \sim d^{\mu_k,y_k}$, $A^\mu_k \sim \muE(\cdot \mid S^\mu_k)$, $S^\nu_k \sim d^{z_k, \nu_k}$, $A^\nu_k \sim \nuE(\cdot \mid S^\nu_k)$. \\
    $g^\mu_k(s,a) = \mu_k(a \mid S^\mu_k)\mathds{1}_{S^\mu_k=s} - \mathds{1}_{A^\mu_k = a}$ \\
    $g^\nu_k(s,a) =  \nu_k(a \mid S^\nu_k)\mathds{1}_{S^\nu_k=s} - \mathds{1}_{A^\nu_k = a}$ \\
    $\mu_{k+1}(a \mid s) \propto \mu_k(a\mid s) \exp\brr{ - \eta g^\mu_k(s,a)}$ \;
    $\nu_{k+1}(b \mid s) \propto \nu_k(b\mid s) \exp\brr{ - \eta g^\nu_k(s,a)}$  
\EndFor
\State \textbf{Return}:$\mu_{\widehat{k}}$, $\nu_{\widehat{k}}$ for $\widehat{k} \sim \mathrm{Unif}([K])$
\end{algorithmic}
\end{algorithm}

\section{Useful results}
For completeness reasons, we provide Theorem 2.1 by~\citet{concentrationofmissingmass} which we used frequently throughout this work as well as a standard binomial concentration result.

\begin{lemma}[Concentration Inequality for Total Variation Distance, see e.g. Thm 2.1 by~\citet{concentrationofmissingmass}]  \label{lemma:l1_concentration}
Let $\gX = \{1, 2, \cdots, |\gX|\}$ be a finite set. Let $P$ be a distribution on $\gX$. Furthermore, let $\widehat{P}$ be the empirical distribution given $m$ i.i.d. samples $x_1, x_2, \cdots, x_n$ from $P$, i.e.,
\begin{align*}
    \widehat{P}(j) = \frac{1}{n} \sum_{i=1}^{n} \mathbb{I} \lb x_i = j \rb.
\end{align*}
Then, with probability at least $1-\delta$, we have that 
\begin{align*}
    \lnorm P - \widehat{P} \rnorm_1 := \sum_{x \in \gX} \labs P(x) - \widehat{P}(x) \rabs \leq \sqrt{\frac{2 |\gX| \log(1/\delta) }{n}}.
\end{align*}
\end{lemma}

\begin{proof}
Define the function $f(x_1, \dots, x_n) = \sum_{x \in \gX} |\widehat{P}(x) - P(x)|$, where $\widehat{P}$ is the empirical distribution. Replacing one sample $x_i$ can change $f$ by at most $2/n$, since the empirical frequencies change by at most $1/n$ per coordinate and total variation sums these differences.

By McDiarmid’s inequality, we have for any $\epsilon > 0$,
\[
\Pr\left( f - \mathbb{E}[f] \geq \epsilon \right) \leq \exp\left(-\frac{n \epsilon^2}{2}\right).
\]

Berend and Kontorovich (2013) show that $\mathbb{E}[f] \leq \sqrt{\frac{|\gX|}{n}}$. Setting the failure probability to $\delta$, we solve
\[
\exp\left(-\frac{n \epsilon^2}{2}\right) = \delta \quad \implies \quad \epsilon = \sqrt{\frac{2 \log(1/\delta)}{n}}.
\]

Therefore, with probability at least $1 - \delta$,
\[
\lnorm P - \widehat{P} \rnorm_1 \leq \sqrt{\frac{|\gX|}{n}} + \sqrt{\frac{2 \log(1/\delta)}{n}} \leq \sqrt{\frac{2 |\gX| \log(1/\delta)}{n}},
\]
\end{proof}

\begin{lemma}[Binomial concentration, see e.g. Lemma A.1 by~\citet{10.5555/3540261.3542359}]
  \label{lemma:binomial-concentration}
  Suppose $N\sim \operatorname{Bin} (n, p)$ where $n\ge 1$ and $p\in[0,1]$. Then with probability at least $1-\delta$, we have
  \begin{align*}
    \frac{p}{N\vee 1} \le \frac{8\log(1/\delta)}{n},
  \end{align*}
  where $N \vee 1 := \max\{1,N\}.$
\end{lemma}
\begin{proof}
    We consider two cases.
    Case 1: $p \leq \frac{8\log(1/\delta)}{n}.$ As $N \vee 1 \geq 1,$ we have $\frac{p}{N\vee 1} \leq p \leq \frac{8\log(1/\delta)}{n}$ almost surely. 
    Case 2: $p > \frac{8\log(1/\delta)}{n}.$ Note, that then $\expect[N] = np > 8\log(1/\delta)$ and by the multiplicative Chernoff bound, for any $0 < \epsilon < 1$ it holds true that
    \[\mathbb{P}\lp N < (1-\epsilon)np\rp \leq \exp \lp-\frac{\epsilon^2}{2} np\rp.\]
    Now, with $\epsilon= \frac12$ we have
    \[\mathbb{P}\lp N < (1-\epsilon)np\rp \leq \exp\lp-\frac{np}{8} \rp \leq \delta.\]
    Therefore, with probability of at least $1-\delta$ it holds $N \geq \frac{np}2$ and therefore on this event also $\frac p{n \vee 1} \leq \frac2n.$ In total we get $\frac{p}{N \vee 1} \leq \frac{8\log(1/\delta)}{n}.$
    Combining both cases completes the proof.
\end{proof}

\end{document}